\documentclass[letterpaper]{article}
\usepackage[preprint]{aaai2027}
\usepackage[hyphens]{url}
\usepackage{graphicx}
\usepackage{natbib}
\usepackage{caption}
\usepackage{amsmath}
\usepackage{amssymb}
\usepackage{booktabs}
\usepackage{algorithm}
\usepackage{algorithmic}
\usepackage{amsthm}
\newtheorem{assumption}{Assumption}[section]
\newtheorem{proposition}{Proposition}[section]
\newcommand{\afcrps}{\mathrm{afCRPS}}
\newcommand{\fcrps}{\mathrm{fCRPS}}
\newcommand{\crps}{\mathrm{CRPS}}

\newcommand{\jsecond}{J_{\mathrm{2nd}}}
\newcommand{\E}{\mathbb{E}}
\newcommand{\R}{\mathbb{R}}
\newcommand{\KL}{D_{\mathrm{KL}}}

\title{DIAL: Dynamic and Informed Adjustment of the Native Stochastic Law for Reliable Weather Ensembles}
\author{
Yucheng Wang\textsuperscript{\rm 1,2},
Haonan Guo\textsuperscript{\rm 1,2},
Tianhao Zhao\textsuperscript{\rm 1,2},
Di Wang\textsuperscript{\rm 1,2},\\
Yu Wu\textsuperscript{\rm 1,2},
Zengmao Wang\textsuperscript{\rm 1},
Bo Du\textsuperscript{\rm 1}
}
\affiliations{
\textsuperscript{\rm 1}Wuhan University\\
\textsuperscript{\rm 2}Zhongguancun Academy
}

\begin{document}

\maketitle

\begin{abstract}

Ensemble weather forecasts can look reliable in global averages while hiding too little spread in some cases and too much in others.
These errors vary across variables, spatial scales, lead times, and weather states.
Adapting randomness during forecasting can correct these errors, but training fixes stochastic behavior beforehand and output calibration acts afterward.
Inference guidance also acts during forecasting, but typically scores ensemble members separately rather than the ensemble as a whole.
Reliability feedback arrives at the end, and tracing it backward through many chaotic forecast steps is expensive and unstable.
We therefore introduce \emph{Dynamic and Informed Adjustment of the Native Stochastic Law} (DIAL), which controls a frozen model's native stochasticity.
DIAL learns from past forecasts to predict whether each variable and scale needs more or less spread.
As a forecast develops, it estimates how possible noise changes will affect later forecast accuracy and reliability.
It combines these predictions to make small noise updates at each step.
This loop avoids retraining the weather model and full backpropagation through the forecast.

We evaluate DIAL on FourCastNet 3 (FCN3) using later ERA5 years unseen during training.
DIAL provides better scale-resolved reliability than learned postprocessing and reward-based control baselines.
Relative to the unmodified Native FCN3 reference, DIAL achieves an 8.27\% improvement in the primary multiscale continuous ranked probability score (CRPS).

\end{abstract}

\section{Introduction}

Ensemble weather forecasts represent atmospheric uncertainty by generating multiple plausible futures rather than a single deterministic forecast \cite{price2025gencast,lang2026aifscrps}. An ensemble is reliable when its spread matches the error of its ensemble mean forecast across many cases \cite{gneiting2007calibration,dirkson2026misdiagnosing}. The spread-to-skill ratio (SSR) quantifies this relationship: low values indicate an overconfident ensemble, whereas high values indicate an underconfident ensemble \cite{dirkson2026misdiagnosing}. Yet Figure~\ref{fig:motivation} shows that global averages can hide opposing reliability gaps across variables, spatial scales, forecast lead times, and weather states \cite{rodwell2018flow,jung2008scale}.

\begin{figure}[!t]
    \centering
    \captionsetup{belowskip=-4pt}
    \includegraphics[width=0.84\linewidth]{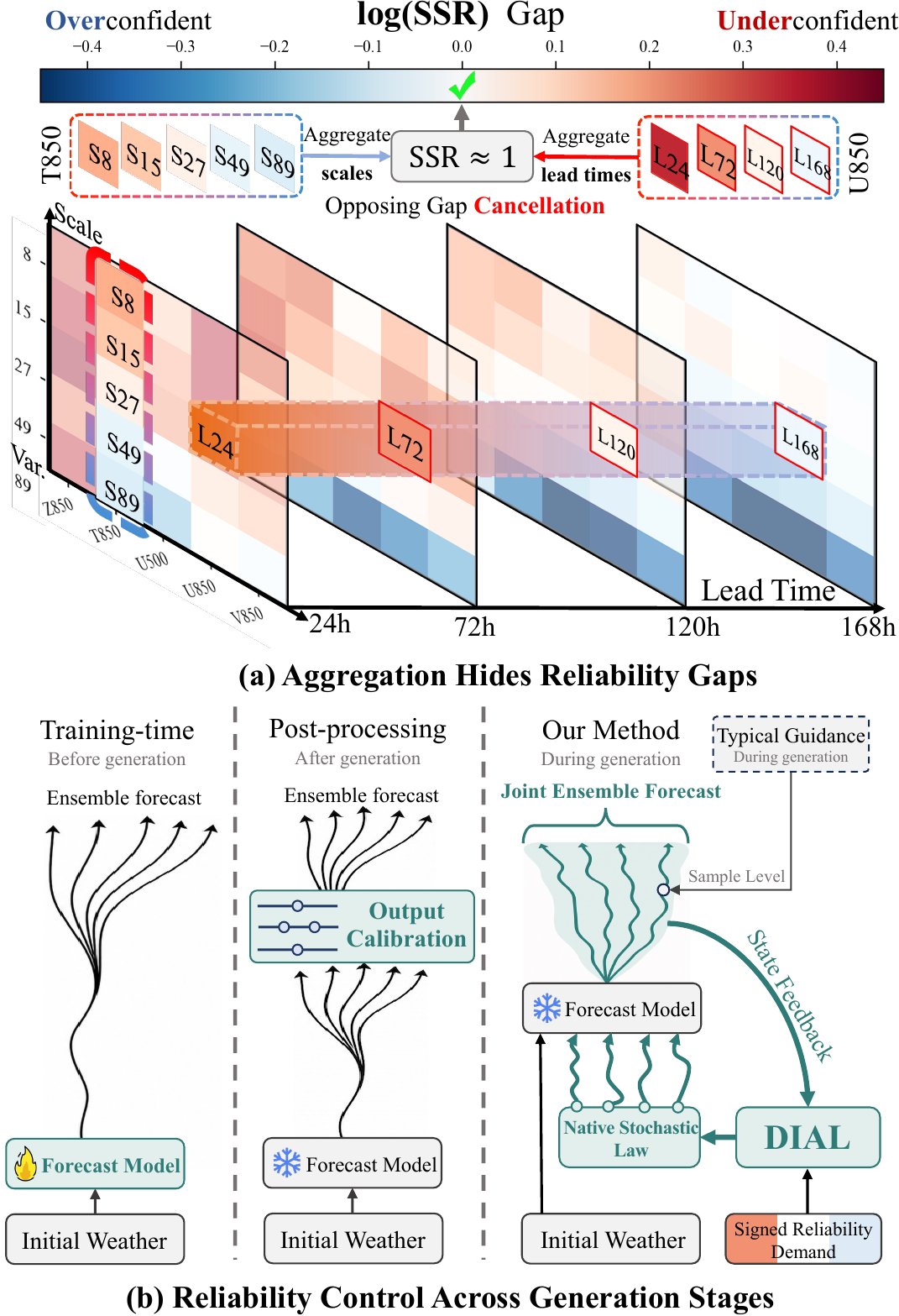}
    \caption{Why pooled reliability is insufficient.
    (a) Aggregation can cancel signed SSR gaps across variables, scales, and
    leads. (b) Unlike training-time stochasticity or output post-processing,
    DIAL uses signed reliability demand and forecast-state feedback to adjust
    the native noise during ensemble generation.}
    \label{fig:motivation}
\end{figure}

Reliability diagnostics reveal gaps in FourCastNet 3 (FCN3) \cite{bonev2025fourcastnet3}, NeuralGCM \cite{kochkov2024neuralgcm}, and ArchesWeatherGen (AWG) \cite{couairon2024archesweathergen}. The patterns differ across systems, and FCN3 exhibits reliability gaps in both directions within a single model. Figure~\ref{fig:motivation}(a) shows that FCN3 is overconfident at large scales but underconfident at the smallest scales. For several wind variables, the gap also changes direction with lead time. Together, these patterns require reliability control to work in both directions, making FCN3 a demanding test case.

Achieving this control requires deciding when and where to intervene in the forecasting process.
Existing methods act at three stages, and each stage has a different limitation.
Training methods learn stochastic behavior together with the forecast model, but this behavior is usually fixed after deployment and cannot adapt to the current forecast state \cite{price2025gencast,lang2026aifscrps,bonev2025fourcastnet3}.
Output calibration corrects completed forecasts, but cannot change how randomness enters the model or develops through weather dynamics \cite{chen2024generativepostprocessing,benbouallegue2024poet}.
Inference guidance steers a pretrained generator during sampling without retraining it \cite{dhariwal2021diffusion,ren2026driftlite}.
Such guidance usually scores samples separately rather than learning how possible updates affect joint ensemble reliability.
These limitations motivate closed-loop control that adjusts the native noise of a frozen model as the forecast evolves and corrects signed reliability gaps.

Implementing this control in FCN3 requires estimating how each possible change to its native noise affects future skill and reliability.
FCN3 injects this noise at every forecast step, so control decisions occur every six hours \cite{bonev2025fourcastnet3}.
Reliability feedback arrives much later.
It can be computed only after the ensemble forecast is complete and the verifying weather data are available.
Full backpropagation through every forecast step and ensemble member would be too expensive.
Chaotic atmospheric dynamics can also make long-range gradients explode or vanish before they reach early control decisions \cite{mikhaeil2022chaotic}.
Our 28-step reliability objective is much longer than FCN3's final 8-step training stage, worsening both problems \cite{bonev2025fourcastnet3}.
Reliable long-horizon control therefore needs a way to assign delayed feedback without backpropagating through the complete forecast.

Recent image generation methods tune sampling with reinforcement learning or stochastic optimal control without backpropagating the final reward through the full process \cite{black2024ddpo,fan2023dpok,domingoenrich2025adjoint}.
Yet their objectives do not transfer directly to ensemble weather reliability.
First, image methods score outputs separately, but reliability is measured over the ensemble.
Second, one reward summarizes image quality, but reliability gaps must remain signed and separate by variable, spatial scale, and lead time.
Combining these gaps can let too little spread cancel too much spread elsewhere.
Third, image objectives usually favor one direction, but forecast spread can be above or below its target.
These differences require an ensemble reliability constraint.
A weather controller must therefore handle delayed feedback and increase or decrease spread separately for each variable and spatial scale.

Figure~\ref{fig:method-overview} summarizes DIAL's learning and control loop.
DIAL treats reliability repair as control of the native noise in a frozen weather model.
It uses signed ensemble constraints to state where spread should increase or decrease without assuming that the future weather distribution is known.
A KL budget limits how far the controlled distribution can move from the model's native distribution.
Together with forecast loss, these constraints define an ideal Gibbs tilt that gives more weight to skillful and reliable forecast paths \cite{donsker1975variational}.
To handle delayed feedback without complete backpropagation, DIAL approximates this tilt through a closed loop.
The Estimator predicts reliability needs, while the Pilot tracks the evolving forecast.
The Q-Critic estimates how possible noise changes affect future skill and reliability.
The Controller then uses the Actuator's variable- and scale-matched noise bases to make allowed updates within the remaining KL budget.
Our contributions are as follows.

\begin{itemize}
    \item We formulate ensemble reliability repair as stochastic optimal control of a frozen model's native noise under signed ensemble and KL constraints.
    \item We propose DIAL, a closed loop that approximates the ideal path law without backpropagation through the full forecast.
    The Estimator and Pilot track reliability needs and forecast evolution, while the Q-Critic and Controller guide variable- and scale-matched Actuator updates.
    
    \item We evaluate DIAL on later ERA5 years against learned postprocessing and reward-based control.
    Relative to the unmodified Native FCN3 reference, DIAL achieves an 8.27\% improvement in the primary multiscale CRPS.
    It also reduces reliability errors in both directions.
\end{itemize}

\section{Related Work}

Probabilistic weather models represent uncertainty with learned ensembles or
stochastic perturbations
\cite{lang2026aifscrps,bonev2025fourcastnet3,berner2017stochastic}.
Their stochastic behavior is usually fixed before deployment, even though
reliability varies with regime, scale, and lead time
\cite{jung2008scale,rodwell2018flow}.
DIAL instead adjusts native noise during forecasting while keeping the model
frozen.

Post-processing calibrates completed ensembles, including their cross-variable
and spatial dependence
\cite{gneiting2005emos,schefzik2013ecc,vannitsem2021postprocessing,
chen2024generativepostprocessing,benbouallegue2024poet}.
These methods cannot change where randomness enters a frozen model or how it
develops through weather dynamics.
DIAL acts inside the forecast.

Generative guidance and reward fine-tuning steer sampling with likelihoods,
rewards, gradients, or policies
\cite{uehara2025inferencealignment,ren2026driftlite,clark2024draft,
black2024ddpo}.
These methods usually favor individual samples in one direction, whereas
ensemble reliability is joint, scale-resolved, and bidirectional.
DIAL therefore controls an ensemble law rather than a preferred sample.

Reliability feedback is delayed, and long-horizon gradients are costly and
unstable in chaotic dynamics \cite{mikhaeil2022chaotic}.
Stochastic-control methods handle delayed effects with value, policy, or
adjoint quantities, while KL regularization limits deviation from a reference
process
\cite{domingoenrich2025adjoint,han2025stochasticcontrol,
blessing2025trustregion}.
Unlike methods that target a supplied reward, DIAL derives signed ensemble
demand and applies budgeted updates in the model's native KL geometry without
full-forecast backpropagation.

\section{Method}

\begin{figure*}[!t]
    \centering
    \captionsetup{belowskip=-4pt}
    \includegraphics[width=0.98\textwidth]{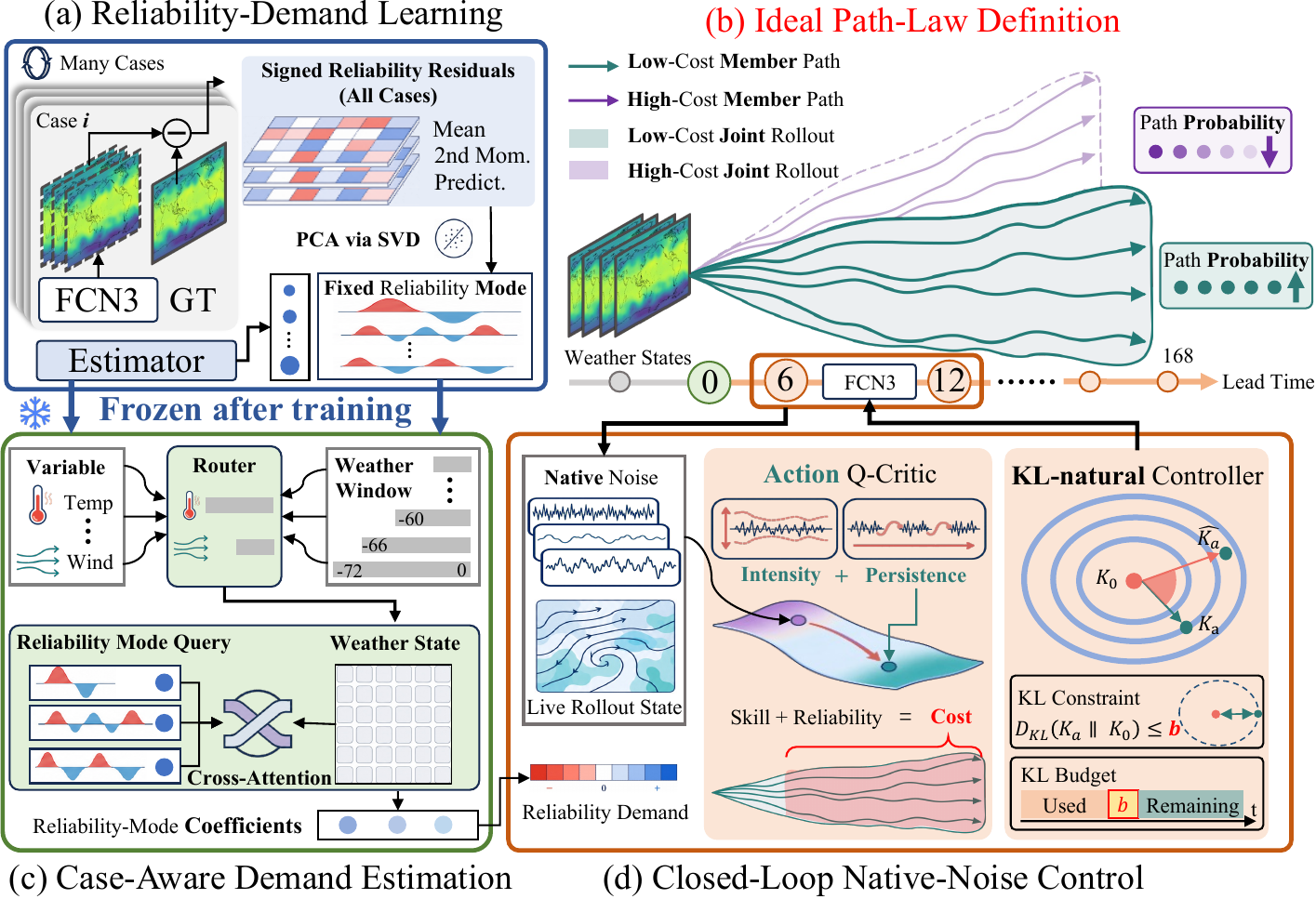}
    \caption{DIAL learns where reliability needs correction and controls native noise as the forecast evolves.
    (a) Training forecasts and verifying weather reveal signed reliability
    errors, which define a fixed set of reliability modes.
    (b) The ideal path law gives more weight to ensemble paths with lower
    combined forecast and reliability cost.
    (c) The frozen Estimator uses recent weather history for each variable and
    scale to predict signed reliability demand for the current case.
    (d) During forecasting, the Pilot provides the live rollout state and the
    Q-Critic scores possible native-noise changes.
    The Controller sends an allowed action to the native-noise Actuator within
    the remaining KL budget.}
    \label{fig:method-overview}
\end{figure*}

\subsection{Problem Formulation and DIAL Overview}

\paragraph{Problem formulation.}
Let $x_0$ be an initial weather state, $Y_{1:T}$ the verifying weather sequence, and $M$ the number of scored ensemble members.  A frozen weather model advances each member $m$ as $X_{t+1}^{(m)}=F_{\rm wx}(X_t^{(m)},Z_{t+1}^{(m)})$, where the stochastic forcing $Z_{t+1}^{(m)}$ is drawn from the model's native noise kernel $K_0$.  We collect the weather states and noise draws of all scored members in $\Omega^{\rm eval}=\{X_{0:T}^{(1:M)},Z_{1:T}^{(1:M)}\}$.  DIAL also runs a Pilot forecast using an independent random stream, denoted by $\Omega^{\rm pilot}$.  The Pilot provides live information about how the controlled forecast is evolving but is excluded from ensemble scoring.

We combine the scored and Pilot paths into the augmented path $\widetilde\Omega=(\Omega^{\rm eval},\Omega^{\rm pilot})$, whose native law is $P_0$.  Feasible adjustments through the Actuator induce controlled laws $P$; we denote their reachable family by $\mathcal P_{\mathcal A}$.  We measure native-law fidelity by the normalized Path-KL $\bar D_{\rm path}(P\Vert P_0):=\nu_{\rm KL}^{-1}\KL(P\Vert P_0)$.  Among these reachable laws, DIAL seeks to improve multiscale probabilistic skill while satisfying ensemble-reliability and native-law fidelity constraints:
\begin{equation}
\begin{aligned}
\min_{P\in\mathcal P_{\mathcal A}}\quad
&\E_P\!\left[\mathcal S_{\rm MS}(\Omega^{\rm eval},Y)\right] \\
\text{s.t.}\quad
&\jsecond(P)\leq\epsilon_{\rm rel},\quad
\bar D_{\rm path}(P\Vert P_0)\leq\beta_{\rm path}.
\end{aligned}
\label{eq:control_problem}
\end{equation}
Here $\mathcal S_{\rm MS}$ is a multiscale forecast score that evaluates ensemble distributions over both full weather fields and fixed spatial-scale projections; its weights are frozen after training.  The normalization $\nu_{\rm KL}$ is also frozen and shared across methods.  The reliability constraint $\jsecond$ compares ensemble moments with verifying weather across a population of forecast cases.  The Path-KL constraint uses no truth and limits how far the controlled stochastic law can depart from the frozen model's native law.  Under absolute continuity, Path-KL decomposes into expected one-step kernel KL terms, allowing DIAL to track and allocate the remaining budget as the forecast evolves.

\paragraph{DIAL overview.}
DIAL tackles this constrained problem with the closed loop in Figure~\ref{fig:method-overview}.
The Estimator predicts signed reliability demand, while the Pilot and recurrent memory track forecast evolution.
The Q-Critic evaluates candidate noise updates, and the Controller uses its gradient and the native KL geometry to obtain a legal Actuator action within the remaining budget.
The following subsections derive the pooled reliability cost and its future estimate, then describe the local update, learning, and inference.

\subsection{From Pooled Reliability to an Ideal Stochastic Law}

Reliability is a property of the jointly generated weather ensemble, not a reward assigned independently to each member.  For a base coordinate $j$ indexing variable, spatial scale, and lead, let $\widetilde X_{imj}$ and $\widetilde Y_{ij}$ be forecast-member and verifying-weather values normalized with training statistics for initialization $i$, and let $\overline X_{ij}=M^{-1}\sum_m\widetilde X_{imj}$.  We measure three complementary signed residuals.
\begin{subequations}
\label{eq:moment_residuals}
The mean row matches the climatological first moment:
\begin{equation}
h^{\rm mean}_{ij}=\overline X_{ij}-\widetilde Y_{ij}.
\label{eq:mean_residual}
\end{equation}
The second-moment row analogously matches the climatological second moment:
\begin{equation}
h^{\rm second}_{ij}
=\frac1M\sum_m\widetilde X_{imj}^{2}-\widetilde Y_{ij}^{2}.
\label{eq:second_residual}
\end{equation}
The predictability row compares member--member and member--truth dependence through ordered pairs:
\begin{equation}
h^{\rm pred}_{ij}
=\frac{\sum_{m\ne m'}\widetilde X_{imj}\widetilde X_{im'j}}
{M(M-1)}-\overline X_{ij}\widetilde Y_{ij}.
\label{eq:predictability_residual}
\end{equation}
\end{subequations}
Stacking the three row types over all base coordinates gives $h_i\in\R^D$, and $W\succ0$ is a training-frozen cell-weight matrix.
\begin{subequations}
\label{eq:pooled_reliability}
DIAL pools the signed residuals across independent forecast initializations as
\begin{equation}
g(P)=\E_i\!\left[
\E_{\Omega\sim P(\cdot\mid x_{0,i})}
\left[h_i(\Omega,Y_i)\right]
\right].
\label{eq:pooled_residual}
\end{equation}
Only after this pooling does DIAL apply the quadratic penalty:
\begin{equation}
\jsecond(P)=g(P)^\top Wg(P).
\label{eq:pooled_functional}
\end{equation}
\end{subequations}
Pooling must precede the penalty because $\jsecond\neq\E_i[h_i^\top Wh_i]$, which would be separable by case rather than a population reliability constraint.

To make this pooled reliability constraint usable by a sequential controller, we dualize its quadratic penalty into a signed cost on each joint ensemble path.  For reliability multiplier $\eta>0$, the quadratic penalty admits the variational form
\begin{equation}
\eta g^\top Wg
=\sup_{\xi\in\R^D}
\left\{\xi^\top g-(4\eta)^{-1}\xi^\top W^{-1}\xi\right\}.
\label{eq:dualized_cost}
\end{equation}
For fixed $\xi$, the dual vector defines the joint ensemble cost $\mathcal C_\xi=\mathcal S_{\rm MS}+\xi^\top h$ and assigns signed weights that favor movement toward pooled reliability and penalize the reverse.  Let $H$ denote the deployment-available history and initial state.
\begin{subequations}
\label{eq:gibbs_components}
Marginalizing the unavailable verifying weather gives
\begin{equation}
\overline{\mathcal C}_\xi(H,\Omega)
=\E_{Y\mid H}\!\left[\mathcal C_\xi(\Omega,Y)\right].
\label{eq:conditional_cost}
\end{equation}
Under the native law, the corresponding normalizer is
\begin{equation}
\mathcal Z_\xi(H)
=\E_{P_0(\cdot\mid H)}\!\left[
\exp\!\left(-\overline{\mathcal C}_\xi/\bar\lambda_{\rm KL}\right)
\right].
\label{eq:native_normalizer}
\end{equation}
\end{subequations}
Eq.~\eqref{eq:control_problem} normalizes Path-KL by $\nu_{\rm KL}$, so the temperature for raw Path-KL is $\bar\lambda_{\rm KL}:=\lambda_{\rm KL}/\nu_{\rm KL}>0$.  For fixed $\xi$, absolute continuity and finite exponential moments give the ideal Gibbs law relative to $P_0$ \cite{donsker1975variational}:
\begin{equation}
\frac{dP_\xi^*(\cdot\mid H)}{dP_0(\cdot\mid H)}(\Omega)
=\frac{\exp[-\overline{\mathcal C}_\xi(H,\Omega)/
\bar\lambda_{\rm KL}]}{\mathcal Z_\xi(H)}.
\label{eq:dual_tilt}
\end{equation}
This law tilts the native stochastic law toward joint ensemble forecasts with lower combined forecast and reliability cost. It is a path-space I-projection toward a reliability set defined by pooled moments \cite{csiszar1975idivergence}, rather than a classical bridge with prescribed endpoint marginals \cite{leonard2014schrodinger}.  The variational identity is standard; DIAL contributes the pooled ensemble reliability constraint and its execution through native weather-noise coordinates.

This ideal law is not directly executable.
Its normalizer integrates over all future joint ensemble forecasts, and its future cost is unknown at decision time.
It may also require noise changes that the Actuator cannot realize.
DIAL therefore avoids the global normalizer and estimates the future cost from forecast-visible state.
At inference, it converts this estimate into a legal local update of the native noise without online truth.

\subsection{Cost-to-Go from Forecast-Visible State}

The ideal law depends on a future cost that is unknown when a control action is chosen.  DIAL constructs a forecast-visible estimate in two stages: the Estimator predicts case-specific reliability demand before forecasting, and the Q-Critic represents the action-conditioned cost-to-go as live Pilot information becomes available.

The Estimator supplies demand rather than an action. From the available pre-forecast weather history, it dynamically selects a context of 3 to 12 frames for each variable and spatial scale, then predicts low-rank coefficients $\widehat\alpha=E_\psi(H^{\rm pre})$.  A training-frozen decoder $B_r$ reconstructs the signed demand field $\widehat r=\bar R+B_r\widehat\alpha$ over moment types, variables, scales, and lead times.  The sign and magnitude of $\widehat r$ describe the direction and size of the predicted reliability residual.  The Estimator runs once at $t_0$.  At forecast step $t$, a deterministic mask $M_t$ zeros out lead times that have already passed, giving the \emph{lead-masked demand} $\widehat r_{t:}=M_t\widehat r$.  This masking does not revise the demand forecast: memory and past actions affect the Q-Critic state, not $\widehat r$.

The online state augments this fixed demand forecast with live information.  The Pilot token $W_t$ summarizes the current controlled weather state without access to verifying weather, while recurrent memory $m_{t-1}$ summarizes earlier Pilot changes and executed noise adjustments.  The state also contains the previous executed action and the exact remaining KL ledger:
\begin{equation}
\begin{aligned}
s_t=(\widehat r_{t:},W_t,t,a_{t-1}^{\rm exec},
m_{t-1},b_t^{\rm rem}).
\end{aligned}
\label{eq:forecast_visible_state}
\end{equation}
\begin{subequations}
\label{eq:q_critic}
A dueling Q-Critic uses separate heads for the remaining multiscale score and low-rank moment contributions.  For each head $\ell\in\{\mathrm{skill},h\}$, subtracting the advantage at the native action $a=0$ gives
\begin{equation}
Q_\phi^\ell(s,a)
=V_\phi^\ell(s)+\widetilde A_\phi^\ell(s,a)
-\widetilde A_\phi^\ell(s,0).
\label{eq:dueling_q}
\end{equation}
The pooled dual combines the two heads into the action-conditioned cost-to-go
\begin{equation}
Q_{\phi,\xi}(s,a)
=Q_\phi^{\rm skill}(s,a)
+(B_r^\top\xi)^\top Q_\phi^h(s,a).
\label{eq:forecast_visible_q}
\end{equation}
\end{subequations}
This construction gives $Q_\phi^\ell(s,0)=V_\phi^\ell(s)$, so the native action has zero advantage.  The moment head $Q_\phi^h\in\R^K$ predicts remaining moment coefficients, while $Q_\phi^{\rm skill}$ predicts the remaining multiscale score.  The pooled dual $\xi$ weights the moment head and preserves the population objective.  Action conditioning lets the Controller use $\nabla_a Q_{\phi,\xi}(s_t,0)$ to choose a local noise update.

\subsection{Projected KL-Natural Control of the Native Kernel}

\paragraph{KL-natural proposal.}
The Q-Critic provides the cost gradient at the native action $a=0$.  The Controller has no learned policy parameters.  Near the native action, the Hessian of the normalized one-step kernel KL in Actuator coordinates defines the native Fisher geometry.  This geometry is computed from the kernel parameterization rather than learned.  A second-order expansion gives
\begin{equation}
\begin{aligned}
\bar D_{\mathrm{KL},t}(a\mid s)
&=\tfrac12 a^\top F_{\rm KL}(s)a
+o(\lVert a\rVert^2),\\
F_{\rm KL}(s)
&=\left.\nabla_a^2\bar D_{\mathrm{KL},t}(a\mid s)
\right|_{a=0}.
\end{aligned}
\label{eq:kl_fisher_geometry}
\end{equation}
The native Fisher penalizes directions that would change the stochastic law most strongly.  Adding a damping term $\gamma I$ gives the damped metric $G_t=\lambda_{\rm KL}F_{\rm KL}(s_t)+\gamma I$ and stabilizes the local solve.  Combining this damped metric with the learned Q-gradient then yields the unconstrained natural proposal
\begin{equation}
a_t^{\rm nat}
=-G_t^{-1}\nabla_aQ_{\phi,\xi}(s_t,0).
\label{eq:natural_proposal}
\end{equation}
Only this unconstrained proposal is closed form; the complete Controller also enforces feasibility explicitly.

\paragraph{Feasible Actuator execution.}
A validation-frozen rule first allocates the one-step allowance $b_t^{\rm step}$ from the remaining ledger $b_t^{\rm rem}$.
\begin{subequations}
\label{eq:feasible_control}
The intersection of the legal action set and this KL allowance defines
\begin{equation}
\mathcal F_t
=\mathcal A_{\rm legal}\cap
\mathcal B_{\rm KL}(b_t^{\rm step}).
\label{eq:feasible_set}
\end{equation}
DIAL projects the natural proposal onto this set in the damped Fisher metric:
\begin{equation}
a_t^{\rm exec}
 =\Pi^{G_t}_{\mathcal F_t}(a_t^{\rm nat}).
\label{eq:feasible_projection}
\end{equation}
\end{subequations}
\begin{subequations}
\label{eq:actuator_dynamics}
The evaluated Actuator controls the innovation scale $\sigma_{k,t}$ and persistence $\rho_{k,t}$ of each native source $k$.  Each source evolves according to
\begin{equation}
Z'_{k,t+1}
 =\rho_{k,t}Z'_{k,t}
+\sigma_{k,t}B_k^0\varepsilon_{k,t+1}.
\label{eq:source_evolution}
\end{equation}
The corresponding Actuator coordinates are
\begin{equation}
a_{k,t}
 =\bigl[\Delta\log\sigma_{k,t},\;
\Delta\operatorname{atanh}\rho_{k,t}\bigr].
\label{eq:actuator}
\end{equation}
\end{subequations}
Here $B_k^0$ maps a standardized innovation into the fixed native spatial structure of source $k$.  The two coordinates change innovation scale and persistence.  Eight sources yield 16 control coordinates, while richer matched bases are outside the evaluated instance.  In this instance, $\mathcal A_{\rm legal}$ bounds $\sigma_{k,t}$ and $\rho_{k,t}$, and $\mathcal B_{\rm KL}(b_t^{\rm step})$ enforces the step allowance.  The map from Actuator coordinates to the kernel is nonlinear, so DIAL recomputes the exact KL of the executed kernel and backtracks toward the native action until both the step allowance and cumulative budget hold.  This gives a Gibbs-motivated local approximation within the Actuator-reachable family rather than an exact global solution of Eq.~\eqref{eq:control_problem}.

\subsection{Learning and Closed-Loop Inference}

DIAL learns only the two predictions required by the local solver: reliability demand from the Estimator and action-conditioned future cost from the Q-Critic.  Native weather ensembles paired with verifying analyses provide the Estimator's signed reliability labels.  Normalization and low-rank bases are fitted only on the training split, and the Q-Critic receives out-of-fold Estimator predictions rather than demand estimates fitted to the same forecast cases.  To train the Q-Critic, each pair starts from the same weather state and matched random innovations: one branch keeps the native noise settings, whereas the other executes an exploratory adjustment.  The remaining-forecast score and moment contributions supervise the two heads.  Differences between the paired branches isolate the action effect and check whether the learned action gradient points toward lower cost.  After each pooled training batch, DIAL increases the reliability multiplier $\eta$ when the pooled reliability constraint is violated and increases the fidelity multiplier $\lambda_{\rm KL}$ when the Path-KL budget is exceeded.  The signed moment weights $\xi$ are updated from the pooled reliability residual rather than from individual forecast cases.  After validation, DIAL freezes the selected checkpoint, dual variables, action bounds, and budget rules; all training-fitted normalizers and low-rank bases also remain fixed.

At inference, all learned components and control rules are frozen, and no verifying weather is available.  The Estimator runs once from the pre-forecast history.  At each forecast step, the Controller uses the lead-masked demand, live Pilot state, recurrent memory, and remaining KL ledger to compute the next feasible action, as summarized in Algorithm~\ref{alg:dial}.

\begin{algorithm}[t]
\caption{DIAL closed-loop inference.}
\label{alg:dial}
\begin{algorithmic}[1]
\REQUIRE Frozen $F_{\rm wx},E_\psi,Q_\phi,K_0,B_r,\bar R$, Actuator, duals, and budget rules; $H^{\rm pre},x_0,\beta_{\rm path}$
\STATE Predict $\widehat r\leftarrow\bar R+B_rE_\psi(H^{\rm pre})$; initialize scored members, Pilot, memory, and the KL ledger
\FOR{$t=0,\ldots,T-1$}
    \STATE Mask $\widehat r_{t:}\leftarrow M_t\widehat r$ and form $s_t$ by Eq.~\eqref{eq:forecast_visible_state}
    \STATE Compute $\nabla_aQ_{\phi,\xi}(s_t,0)$ and $F_{\rm KL}(s_t)$; obtain $a_t^{\rm nat}$ by Eq.~\eqref{eq:natural_proposal}
    \STATE Allocate $b_t^{\rm step}$, project to $a_t^{\rm exec}$, and backtrack with the exact KL of the executed kernel
    \STATE Advance the Pilot and scored members under $K_{a_t^{\rm exec}}$ with independent streams
    \STATE Update memory from the Pilot change and executed action; update the KL ledger
\ENDFOR
\RETURN the $M$ scored ensemble members
\end{algorithmic}
\end{algorithm}

\begin{figure*}[!t]
    \centering
    \includegraphics[width=\textwidth]{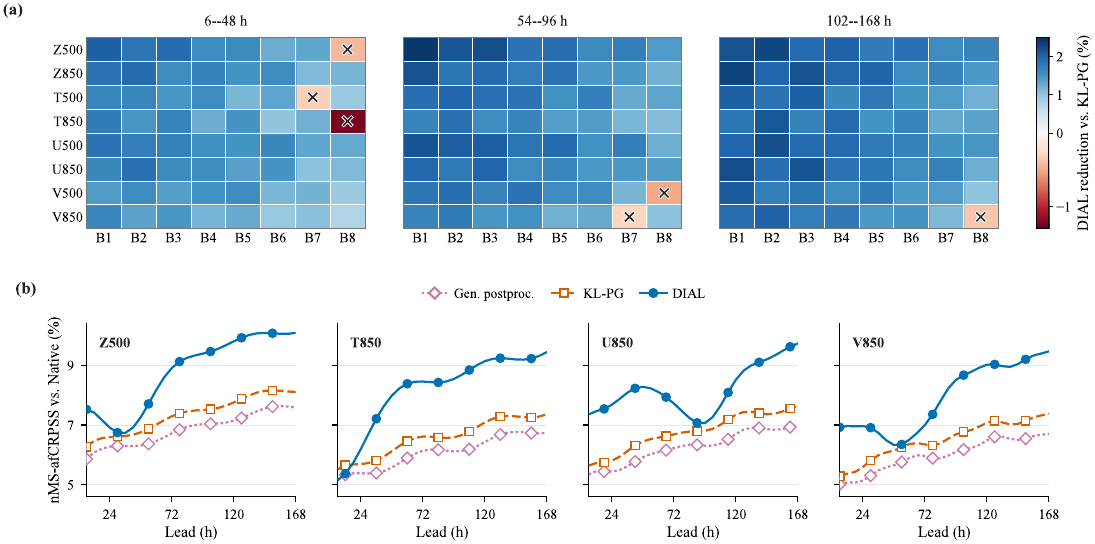}
    \caption{Scale- and lead-resolved nMS-afCRPS.
    (a) DIAL's relative reduction over KL-PG across the reported fields, bands B1--B8
    (largest to smallest), and three lead windows; positive values favor DIAL
    and $\times$ marks nonpositive cells. (b) Native-relative skill across 28
    six-hour leads for DIAL and the strongest output-side and reward-control
    baselines, shown for Z500, T850, U850, and V850. Lead curves for all
    reported fields appear in the supplement.}
    \label{fig:main-results}
\end{figure*}

\begin{figure*}[!t]
    \centering
    \includegraphics[width=\textwidth]{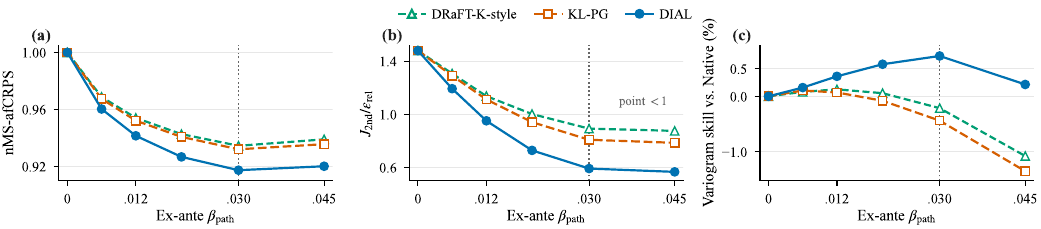}
    \caption{Budget robustness.
    DIAL and the reward-control baselines use the same validation-frozen
    ex-ante Path-KL grid. Panels report (a) nMS-afCRPS, (b) pooled reliability
    $\jsecond/\epsilon_{\rm rel}$, and (c) Native-relative fixed-pair variogram
    skill. Dotted and dashed lines mark the registered budget and point
    reliability threshold, respectively.}
    \label{fig:budget-robustness}
\end{figure*}

\section{Experiments}

\subsection{Experimental Setup}

\paragraph{Data and Forecast Task.}
We evaluate DIAL on frozen FCN3 forecasts initialized and verified with ERA5. DIAL applies native-noise control across all FCN3 variables throughout each rollout. We evaluate geopotential, temperature, and both wind components at 500 and 850 hPa, and include all eight fields in every aggregate. Scores cover eight spectral bands and 28 six-hour leads through 168 h. Training uses 240 initializations from 2017--2020, and validation uses 60 from 2021. The sealed test set contains 60 initializations from 2022, each with 30 ensemble members. We freeze all learned components, metric choices, action bounds, budgets, and statistical rules before test evaluation. Checkpoint, initialization sampling, preprocessing, and regridding details appear in the supplement.

\paragraph{Baselines.}
We compare seven ensemble methods with equal member weights. Native FCN3 is the unmodified reference. Three baselines act after the ensemble forecast is generated. Cond-SP uses forecast-visible state and estimated reliability demand to correct spread by variable, band, and lead while preserving the ensemble mean. PoET is retrained on Native FCN3 ensembles and retains its member-set processing and ensemble size \cite{benbouallegue2024poet}. Generative multivariate postprocessing is adapted to sample complete forecast paths \cite{chen2024generativepostprocessing}. DIAL is compared with two generation-time controllers. DRaFT-K-style and KL-PG share DIAL's forecast-visible state, native-noise actions, legal set, objective, and training budget. All three use the same ex-ante Path-KL budget, chosen on validation. DRaFT-K-style differentiates through the final $K$ steps, with $K$ selected on validation \cite{clark2024draft}. KL-PG uses at most a capacity-matched scalar value baseline and excludes DIAL's action-conditioned Q-Critic and KL-natural solver \cite{black2024ddpo,fan2023dpok}. Detailed information, compute, and output-distortion ledgers appear in the supplement.

\paragraph{Evaluation Protocol.}
nMS-afCRPS with $\alpha=0.95$ is primary, with training-frozen normalizers, bands, and weights \cite{lang2026aifscrps}. We also report full-field fCRPSS, Native-relative ensemble-mean RMSE improvement, RMS($C$) with $C=-2\log\mathrm{SSR}_{\rm finite\text{-}M}$, $\jsecond/\epsilon_{\rm rel}$, and realized Path-KL; further diagnostics appear in the supplement. Dates, not members, are the statistical units. Two-sided 95\% confidence intervals resample paired date blocks; common member seeds preserve method pairing. Each learned method uses three training seeds. At the validation-frozen budget, one-sided contrasts with the four learned baselines use Holm correction at familywise level 0.05.

\subsection{Main Results}

\begin{table}[!t]
\centering
\captionsetup{skip=7pt,belowskip=-2pt}
\small
\setlength{\tabcolsep}{2.2pt}
\renewcommand{\arraystretch}{1.05}
\begin{tabular}{@{}lrrrr@{}}
\toprule
Method & \shortstack{nMS-\\afCRPS $\downarrow$}
& \shortstack{fCRPSS\\(\%) $\uparrow$}
& \shortstack{RMSE gain\\(\%) $\uparrow$}
& \shortstack{RMS($C$)\\$\downarrow$} \\
\midrule
Native FCN3 & 1.0000 & 0.000 & 0.000 & 0.4353 \\
Cond-SP & 0.9681 & 2.737 & 0.018 & 0.3207 \\
\midrule
PoET & 0.9453 & 4.824 & 1.274 & 0.2491 \\
Gen.\ postproc. & 0.9370 & 5.913 & \textbf{2.413} & 0.2278 \\
\midrule
DRaFT-K-style & 0.9345 & 5.663 & 1.537 & 0.2436 \\
KL-PG & 0.9320 & 6.050 & 1.780 & 0.2310 \\
\textbf{DIAL} & \textbf{0.9173} & \textbf{7.543} & 2.117 & \textbf{0.1794} \\
\bottomrule
\end{tabular}
\caption{Aggregate equal-weight ensemble performance.
DIAL applies native-noise control across all FCN3 variables, while scores
aggregate all eight evaluated fields ($Z$, $T$, $U$, and $V$ at 500 and 850 hPa).
nMS-afCRPS uses $\alpha=0.95$ and Native normalization;
fCRPSS and RMSE improvement are relative to Native FCN3; RMS($C$)=0 is ideal.}
\label{tab:main-results}
\end{table}

\paragraph{Overall Forecast Quality.}
At the validation-frozen budget, DIAL reduces nMS-afCRPS by 8.27\% from Native FCN3 and by 1.57\% from KL-PG, the strongest baseline (95\% CI [0.05, 3.10]\%; Holm-adjusted $p=0.022$). All four are significant. DIAL has the highest fCRPSS and lowest RMS($C$), but generative postprocessing gives the largest improvement in ensemble-mean RMSE.

\paragraph{Variation across Scale and Lead.}
The aggregate gain is broad but not uniform. Figure~\ref{fig:main-results}(a) favors DIAL in 186 of 192 reported field--band--window cells, with six local exceptions concentrated at the shortest spatial scales. The representative lead curves in Figure~\ref{fig:main-results}(b) show the same advantage over most of the seven-day forecast. The two short-range T850 points do not improve over KL-PG, so the result does not rely on a claim of universal cellwise dominance.

\paragraph{Reliability within the Native-Law Budget.}

At the registered ex-ante budget $\beta_{\rm path}=0.030$, DIAL realizes a Path-KL of 0.0259 and gives $\jsecond/\epsilon_{\rm rel}=0.591$, with a one-sided 95\% upper confidence bound of 0.79. Since this bound is below one, DIAL passes the reliability gate at the same operating point, while ensemble-mean RMSE remains better than Native. We examine budget effects below.

\subsection{Ablation Studies}

\paragraph{Budget Robustness.}
Figure~\ref{fig:budget-robustness} varies the ex-ante budget while holding all components fixed. DIAL gives lower nMS-afCRPS than both reward-control baselines at every nonzero budget and crosses the point threshold at a smaller budget. Beyond the registered point, additional budget lowers $\jsecond$ but worsens nMS-afCRPS and variogram skill.

\begin{table}[!t]
\centering
\captionsetup{skip=7pt,belowskip=-2pt}
\footnotesize
\setlength{\tabcolsep}{1.4pt}
\renewcommand{\arraystretch}{1.02}
\begin{tabular}{@{}lrrrr@{}}
\toprule
Variant & \shortstack{nMS-\\afCRPS $\downarrow$} & \shortstack{RMSE gain\\(\%) $\uparrow$} & \shortstack{$\jsecond/$\\$\epsilon_{\rm rel}$ $\downarrow$} & \shortstack{Realized\\Path-KL} \\
\midrule
\multicolumn{5}{@{}l}{\textit{(A) Forecast-visible information}} \\
No Estimator & 0.9449 & 1.084 & 0.874 & 0.0238 \\
Matched stale-weather & 0.9283 & 1.736 & 0.719 & 0.0269 \\
\textbf{Full DIAL} & 0.9173 & 2.117 & 0.591 & 0.0259 \\
\midrule
\multicolumn{5}{@{}l}{\textit{(B) Control-mechanism failures}} \\
Myopic future cost & 0.9286 & 1.423 & 0.772 & 0.0246 \\
Euclidean update & 0.9252 & 1.834 & 0.682 & 0.0291 \\
No $\jsecond$ & 0.9133 & 2.284 & 1.137$^\dagger$ & 0.0267 \\
\textbf{Full DIAL} & 0.9173 & 2.117 & 0.591 & 0.0259 \\
\bottomrule
\end{tabular}
\caption{Ablations at the registered Path-KL budget.
Panel A tests forecast-visible information; Panel B tests control mechanisms.
Ratios are point estimates; $\dagger$ marks values above $1$. Formal
feasibility uses a one-sided 95\% upper confidence bound.
Metrics use each row's final executed kernel.}
\label{tab:ablations}
\end{table}

\paragraph{Forecast-Visible Information.}
Panel A of Table~\ref{tab:ablations} separates case-specific demand from live weather feedback. Replacing Estimator demand with a training-only climatology or reusing the initial weather token at every step degrades nMS-afCRPS relative to full DIAL. Both variants retain the Controller, legal action space, and ex-ante budget. These results support both information sources without asserting an interaction.

\paragraph{Control-Mechanism Failures.}
Panel B isolates three mechanisms. A myopic cost degrades 102--168 h nMS skill, matching the preregistered long lead signature. This supports modeling how early actions affect later forecasts. The Euclidean update uses more Path-KL yet gives worse nMS-afCRPS, supporting the native Fisher geometry. Removing $\jsecond$ improves the primary score, but the reliability ratio rises above one. A better primary score alone therefore does not satisfy the pooled-reliability criterion.

\section{Conclusion}

DIAL adapts ensemble stochasticity during generation rather than calibrating completed forecasts. It converts signed reliability gaps across scales into future costs and adjusts native noise through KL-natural updates within a Path-KL budget. On frozen FCN3, DIAL improves multiscale skill and pooled reliability over learned postprocessing and reward-based control at the registered budget, although not every metric improves. Its local control is limited to the Actuator's noise coordinates and requires a Pilot forecast and multi-member scoring. DIAL therefore provides an auditable reliability-aware stochastic closure without retraining the forecaster.

\clearpage
\setcounter{secnumdepth}{2}
\setcounter{section}{0}
\setcounter{subsection}{0}
\setcounter{equation}{0}
\setcounter{figure}{0}
\setcounter{table}{0}
\numberwithin{equation}{section}
\renewcommand{\thefigure}{S\arabic{figure}}
\renewcommand{\thetable}{S\arabic{table}}

\section*{Scope and Reading Guide}

The main paper states the complete problem, method, and evidence required for
its central claims. This document supplies derivations, protocol definitions,
and extended diagnostic views. Sections~A and B formalize the pooled
reliability objective, ideal path law, and forecast-visible approximation.
Section~C fixes the comparison and aggregation rules used by the extended
results. Sections~D--F report complete field coverage, signed reliability
response, budget behavior, and mechanism audits. The notation follows the main
paper; repeated definitions are included where they are needed to make a proof
or table independently readable.

\appendix

\section{Pooled Reliability and the Ideal Path Law}
\label{sec:supp-ideal-law}

\subsection{Augmented Joint-Path Setup and Assumptions}

Let $H$ denote the weather information available when a forecast is
initialized and let $Y_{1:T}$ be the verifying weather sequence. The frozen
forecaster advances scored member $m$ according to
\begin{equation}
X_{t+1}^{(m)}=F_{\rm wx}\!\left(X_t^{(m)},Z_{t+1}^{(m)}\right),
\qquad Z_{t+1}^{(m)}\sim K_0(\cdot\mid H_t^{(m)}).
\label{eq:supp-native-transition}
\end{equation}
The $M\geq2$ scored paths form $\Omega^{\rm eval}$. A Pilot path
$\Omega^{\rm pilot}$ driven by an independent random stream supplies live
forecast information but never enters an ensemble score. The theoretical path
variable is therefore
$\widetilde\Omega=(\Omega^{\rm eval},\Omega^{\rm pilot})$. Its conditional
native law is $P_0(\cdot\mid H)$, and legal Actuator updates induce a reachable
family $\mathcal P_{\mathcal A}$ of controlled laws.

The derivations use the following conditions.

\begin{assumption}[Information and conditional sampling]
The native and controlled systems share the law of $(H,Y)$; denote their common
history marginal by $\mu_H$. Every admissible deployed policy is a deterministic
measurable function of the forecast filtration generated by $H$ and the
forecast paths observed so far; stochasticity comes from the native scored and
Pilot innovations already contained in $\widetilde\Omega$. Conditional on
$H$, all scored and Pilot innovation streams are sampled independently of the
verifying weather $Y$. Consequently, for every admissible conditional path law,
\begin{equation}
P(d\omega,dy\mid H=h)=P(d\omega\mid H=h)Q(dy\mid H=h),
\label{eq:supp-conditional-factorization}
\end{equation}
where $Q$ is the common verifying-weather kernel. Neither the Pilot nor the
Controller observes $Y$ at inference.
\label{assump:information}
\end{assumption}

\begin{assumption}[Absolute continuity]
Every considered controlled transition is absolutely continuous with respect
to the corresponding native augmented transition. Consequently,
$P\ll P_0$ for every $P\in\mathcal P_{\mathcal A}$ with finite Path-KL.
\label{assump:absolute-continuity}
\end{assumption}

\begin{assumption}[Moment and exponential integrability]
The multiscale score and moment residuals are absolutely integrable under every
law in the variational domain. For each fixed dual vector $\xi$ and positive
temperature $\bar\lambda_{\rm KL}$ considered below, the following conditions
hold for $\mu_H$-almost every $h$; no parameter-uniform common null set is
claimed. The conditional potential $\overline{\mathcal C}_\xi(h,\cdot)$ is
finite $P_0(\cdot\mid h)$-almost surely and
\begin{equation}
0<\mathcal Z_\xi(h)<\infty,
\qquad
\E_{P_0(\cdot\mid h)}\!\left[
e^{-\overline{\mathcal C}_\xi/\bar\lambda_{\rm KL}}
\bigl|\overline{\mathcal C}_\xi\bigr|\right]<\infty.
\label{eq:supp-tilted-integrability}
\end{equation}
For each such $(h,\xi,\bar\lambda_{\rm KL})$, the variational domain is exactly
\begin{equation}
\mathfrak D_{\xi}(h)=
\left\{P_h\ll P_{0,h}\ \middle|\ 
\begin{aligned}
\KL(P_h\Vert P_{0,h})&<\infty,\\
\E_{P_h}|\overline{\mathcal C}_\xi(h,\Omega)|&<\infty
\end{aligned}
\right\}.
\label{eq:supp-variational-domain}
\end{equation}
Thus the exponentially tilted law belongs to the domain under
Equation~\eqref{eq:supp-tilted-integrability}; the domain is not an additional
restricted model family.
\label{assump:integrability}
\end{assumption}

\begin{assumption}[Frozen reliability geometry]
The cell-weight matrix $W$ is symmetric positive definite. All projections,
normalizers, cell weights, and the Path-KL normalization $\nu_{\rm KL}>0$ are
fitted before evaluation and then held fixed.
\label{assump:frozen-geometry}
\end{assumption}

Regard $P$ and $P_0$ as joint laws on $(H,\widetilde\Omega)$ with the same
$H$-marginal. For $P\in\mathcal P_{\mathcal A}$, the conditional chain rule
then gives the normalized native-law deviation
\begin{equation}
\begin{aligned}
D_H(P\Vert P_0)
&=\KL\!\left(P(\cdot\mid H)\Vert P_0(\cdot\mid H)\right),\\
\bar D_{\rm path}(P\Vert P_0)
&\equiv\nu_{\rm KL}^{-1}\KL(P\Vert P_0)
=\nu_{\rm KL}^{-1}\E_H[D_H(P\Vert P_0)].
\end{aligned}
\label{eq:supp-normalized-path-kl}
\end{equation}
This quantity measures deviation after legality has been enforced. It does not
replace stability, support, covariance, or balance constraints imposed by the
Actuator.

\subsection{Finite-Member Scores and the Signed Spread--Error Gap}

For a scalar forecast ensemble $x_{1:M}$ and observation $y$, the
almost-fair score is
\begin{equation}
\begin{aligned}
\ell_{\alpha}(x_m,x_{m'};y)
={}&|x_m-y|+|x_{m'}-y|\\
&-\left(1-\frac{1-\alpha}{M}\right)|x_m-x_{m'}|,\\
\afcrps_{\alpha}(x_{1:M},y)
={}&\frac{1}{M(M-1)}\sum_{m<m'}
\ell_{\alpha}(x_m,x_{m'};y).
\end{aligned}
\label{eq:supp-afcrps}
\end{equation}
Equivalently,
$\afcrps_{\alpha}=\alpha\fcrps+(1-\alpha)\crps$, where $\fcrps$
uses the finite-ensemble unbiased member-pair coefficient
\cite{ferro2014fairscores,leutbecher2019ensemblesize}. The reported primary
score fixes $\alpha=0.95$. For variable $v$, lead $\tau$, and fixed spectral
projection $\Pi_b$, the multiscale objective is
\begin{equation}
\begin{aligned}
L_{v,\tau}={}&\omega_0\afcrps(X_{1:M,v,\tau},Y_{v,\tau})\\
&+\sum_b\omega_b\afcrps(
\Pi_bX_{1:M,v,\tau},\Pi_bY_{v,\tau}),\\
\mathcal S_{\rm MS}
={}&\sum_{v,\tau}w_{v,\tau}L_{v,\tau}.
\end{aligned}
\label{eq:supp-ms-afcrps}
\end{equation}
Here $\alpha\in[0,1]$; all $\omega_0$, $\omega_b$, and $w_{v,\tau}$ are
nonnegative, with nonzero total weight at each aggregation level. The weights
and projections are frozen before evaluation. Under these conditions the
objective is a weighted collection of proper marginal or projected scores; it
is not claimed to identify the complete spatiotemporal joint law.

The signed mechanism diagnostic uses the same pooling weights for spread and
ensemble-mean error. For a diagnostic cell $c$, let $\langle\cdot\rangle_c$
denote its normalized frozen average over initializations and valid projected
coordinates, and define
\begin{equation}
\begin{aligned}
E_{{\rm err},c}
&=\left\langle(\overline X-Y)^2\right\rangle_c,\\
E_{{\rm spread},c}^{\rm sample}
&=\left\langle\frac{1}{M-1}\sum_{m=1}^M
(X_m-\overline X)^2\right\rangle_c,\\
E_{{\rm spread},c}^{\rm fair}
&=\frac{M+1}{M}E_{{\rm spread},c}^{\rm sample},\\
C_c&=\log E_{{\rm err},c}-\log E_{{\rm spread},c}^{\rm fair}.
\end{aligned}
\label{eq:supp-signed-c}
\end{equation}
Subsequent formulas suppress the cell subscript. The diagnostic is reported
only for cells with
$E_{{\rm err},c}>0$ and $E_{{\rm spread},c}^{\rm sample}>0$; degenerate cells
are outside this logarithmic diagnostic rather than assigned an implicit
regularizer.
Thus
$C=-2\log\mathrm{SSR}_{\rm finite\text{-}M}$. Positive $C$ means that
error energy exceeds effective spread energy and therefore calls for more
spread; negative $C$ calls for less. If an intervention changes log error and
log spread amplitudes by $\delta_e$ and $\delta_s$, respectively, then
\begin{equation}
C_{\rm after}=C_{\rm native}+2\delta_e-2\delta_s.
\label{eq:supp-c-response-identity}
\end{equation}
This identity is used in every signed-response audit below.

\subsection{Population Reliability Must Be Pooled Before Squaring}

For initialization $i$ and cell $j$, let $\widetilde X_{imj}$ and
$\widetilde Y_{ij}$ be forecast and truth coordinates normalized by a common
training reference, with
$\overline X_{ij}=M^{-1}\sum_m\widetilde X_{imj}$. The three moment rows are
\begin{equation}
\begin{aligned}
h^{\rm mean}_{ij}
&=\overline X_{ij}-\widetilde Y_{ij},\\
h^{\rm second}_{ij}
&=M^{-1}\sum_m\widetilde X_{imj}^2-\widetilde Y_{ij}^2,\\
h^{\rm pred}_{ij}
&=\frac{\sum_{m\ne m'}\widetilde X_{imj}\widetilde X_{im'j}}
{M(M-1)}-\overline X_{ij}\widetilde Y_{ij}.
\end{aligned}
\label{eq:supp-moment-rows}
\end{equation}
The predictability row uses ordered member pairs and excludes diagonal terms.
Stacking every valid row gives $h_i\in\R^D$. Equivalently, let
$r_{\rm mom}(h,\omega,y)$ denote that stacked random vector before a case is
realized. DIAL first pools signed residuals over the common history, controlled
forecast, and verifying-weather laws,
\begin{equation}
\begin{aligned}
g(P)&=\int\mu_H(dh)\int P(d\omega\mid h)\\
&\qquad{}\times\int Q(dy\mid h)\,r_{\rm mom}(h,\omega,y),\\
\jsecond(P)&=g(P)^\top Wg(P).
\end{aligned}
\label{eq:supp-pooled-reliability}
\end{equation}
Absolute integrability in Assumption~\ref{assump:integrability} makes the
iterated integral and the equivalent expectation order well defined.
In general, $\jsecond(P)\ne\E_i[h_i^\top Wh_i]$. The right-hand side is a
casewise risk and would penalize weather-to-weather variability even when the
population moment is correct. Equation~\eqref{eq:supp-pooled-reliability}
instead expresses a population reliability constraint.

\subsection{Quadratic Dualization and the Gibbs Law}

\begin{proposition}[Quadratic reliability dual]
Under Assumption~\ref{assump:frozen-geometry}, for $\eta>0$ and any
$g\in\R^D$,
\begin{equation}
\eta g^\top Wg
=\sup_{\xi\in\R^D}\left\{
\xi^\top g-\frac{1}{4\eta}\xi^\top W^{-1}\xi
\right\},
\label{eq:supp-quadratic-dual}
\end{equation}
and the unique maximizer is $\xi^*=2\eta Wg$.
\label{prop:quadratic-dual}
\end{proposition}

\begin{proof}
Writing $\Delta=\xi-2\eta Wg$, completing the square gives
\begin{equation}
\begin{aligned}
\xi^\top g-\frac{1}{4\eta}\xi^\top W^{-1}\xi
={}&\eta g^\top Wg\\
&-\frac{1}{4\eta}\Delta^\top W^{-1}\Delta.
\end{aligned}
\end{equation}
Positive definiteness makes the last term nonpositive and zero only at
$\xi=2\eta Wg$.
\end{proof}

The notation $r_{\rm mom}$ also separates the residual vector from a fixed
history value $h$. For fixed $\xi$, define the joint-ensemble cost
$\mathcal C_\xi=\mathcal S_{\rm MS}+\xi^\top r_{\rm mom}$. Since future truth is not
available to the deployed controller, its forecast-side potential is
\begin{equation}
\begin{aligned}
\overline{\mathcal C}_\xi(h,\omega)
&=\int\mathcal C_\xi(\omega,y)Q(dy\mid h)\\
&=\E_{Y\sim Q(\cdot\mid h)}[\mathcal C_\xi(\omega,Y)].
\end{aligned}
\label{eq:supp-conditional-cost}
\end{equation}
By absolute integrability, averaging this cost under $(\mu_H,P,Q)$ gives
$\E_P[\mathcal S_{\rm MS}]+\xi^\top g(P)$.
Equation~\eqref{eq:supp-conditional-factorization} is what permits this
truth-only conditional expectation while retaining $\omega$ as the argument to
be tilted.

\begin{proposition}[Ideal conditional path law]
Suppose Assumptions~\ref{assump:information}--\ref{assump:integrability}
hold and let $\bar\lambda_{\rm KL}>0$. For $\mu_H$-almost every fixed history
$H=h$, among all laws $P_h\in\mathfrak D_\xi(h)$, the unique
minimizer of
\begin{equation}
\E_{P_h}[\overline{\mathcal C}_\xi(h,\Omega)]
+\bar\lambda_{\rm KL}\KL(P_h\Vert P_{0,h})
\label{eq:supp-kl-regularized-objective}
\end{equation}
has density
\begin{equation}
\begin{aligned}
\frac{dP_{\xi,h}^*}{dP_{0,h}}(\omega)
&=\frac{\exp[-\overline{\mathcal C}_\xi(h,\omega)/
\bar\lambda_{\rm KL}]}{\mathcal Z_\xi(h)},\\
\mathcal Z_\xi(h)
&=\E_{P_{0,h}}\exp[-\overline{\mathcal C}_\xi(h,\Omega)/
\bar\lambda_{\rm KL}].
\end{aligned}
\label{eq:supp-gibbs-law}
\end{equation}
For any restricted conditional family
$\mathcal P_{\mathcal A}(h)\cap\mathfrak D_\xi(h)$, the regularized objective
and $\KL(P_h\Vert P_{\xi,h}^*)$ have the same infimizers whenever the infimum is
attained. More precisely, the restricted infima obey
\begin{equation}
\inf_{P_h} J_h(P_h)
=\bar\lambda_{\rm KL}\inf_{P_h}\KL(P_h\Vert P_{\xi,h}^*)
-\bar\lambda_{\rm KL}\log\mathcal Z_\xi(h),
\label{eq:supp-restricted-infima}
\end{equation}
where $J_h$ denotes Equation~\eqref{eq:supp-kl-regularized-objective} and both
infima range over the same restricted domain. No attainment or
global optimality is asserted for a nonclosed, nonconvex reachable family.
\label{prop:gibbs-law}
\end{proposition}

\begin{proof}
Fix an admissible $h$. For any $P_h\ll P_{0,h}$ in the stated domain,
substitute Equation~\eqref{eq:supp-gibbs-law} into the Radon--Nikodym
derivative:
\begin{equation}
\begin{aligned}
\bar\lambda_{\rm KL}\KL(P_h\Vert P_{\xi,h}^*)
={}&\E_{P_h}[\overline{\mathcal C}_\xi(h,\Omega)]\\
&+\bar\lambda_{\rm KL}\KL(P_h\Vert P_{0,h})\\
&+\bar\lambda_{\rm KL}\log\mathcal Z_\xi(h).
\end{aligned}
\label{eq:supp-kl-decomposition}
\end{equation}
The last term is independent of $P_h$. Assumption~\ref{assump:integrability}
makes every term finite on the domain and makes the tilted law admissible.
Nonnegativity of KL proves that $P_{\xi,h}^*$ minimizes
Equation~\eqref{eq:supp-kl-regularized-objective}, with uniqueness up to
$P_{0,h}$-null sets. Applying the same identity on a restricted family proves
the infimum and infimizer statement. This is the Gibbs variational identity
\cite{donsker1975variational,csiszar1975idivergence}.
\end{proof}

For fixed $\xi$, the ideal law specifies a full-path skill-and-moment tilt
relative to the native process. It does not prescribe a terminal marginal and
is therefore not a classical fixed-endpoint bridge
\cite{leonard2014schrodinger}. Proposition~\ref{prop:quadratic-dual} alone does
not establish a saddle-point interchange, an optimal multiplier, or projection
onto the zero-moment reliability set. The reachable family can also be
nonconvex, so the practical controller is a receding local approximation to the
fixed-dual Gibbs objective, not an exact global reliability projection.
For precision, despite the main paper's compact ``I-projection toward
reliability'' shorthand, $P_{\xi,h}^*$ is not literally asserted to be a KL
projection onto a reliability set. The formal result is only the fixed-$\xi$
Gibbs identity above: it does not assert that $P_{\xi,h}^*$ belongs to the
zero-moment reliability set or solves the original constrained problem.

\subsection{Path-KL Accounting and the Local Projection Identity}

\begin{proposition}[Augmented-kernel KL chain rule]
Suppose the native and controlled augmented processes have the same initial
law. Let $\mathcal G_t$ be the complete pre-transition filtration of
$\widetilde\Omega$, including the scored-member states, Pilot and
stochastic-source states, and earlier actions.
If the regular conditional kernels
$K_{0,t}^{\rm aug}(\cdot\mid\mathcal G_t)$ and
$K_{a_t,t}^{\rm aug}(\cdot\mid\mathcal G_t)$ satisfy
Assumption~\ref{assump:absolute-continuity}, then, in the common finite or
extended-value sense,
\begin{equation}
\KL(P\Vert P_0)
=\E_P\sum_{t=0}^{T-1}
\KL\!\left(K_{a_t,t}^{\rm aug}(\cdot\mid\mathcal G_t)
\Vert K_{0,t}^{\rm aug}(\cdot\mid\mathcal G_t)\right).
\label{eq:supp-kl-chain-rule}
\end{equation}
\label{prop:kl-chain-rule}
\end{proposition}

\begin{proof}
The path-density ratio factorizes into the finite product of full-history
conditional transition-density ratios because the initial laws coincide.
Taking its logarithm gives a finite sum over forecast steps. Conditional
expectation under $P$, followed by the tower property, converts each integrable
summand into the conditional kernel KL in
Equation~\eqref{eq:supp-kl-chain-rule}; the standard nonnegative extended-value
chain rule covers the infinite-KL case.
\end{proof}

The augmented kernel contains all scored members and the Pilot. Writing a
kernel as $K(\cdot\mid s_t)$ is valid only when $s_t$ is a sufficient controlled
Markov state for $\mathcal G_t$; the accounting theorem itself does not require
that reduction. The joint kernel may be further decomposed memberwise only
when the declared conditional independence structure is valid. Shared actions
and Pilot feedback remain part of the conditioning information even in that
special case. This prevents the Pilot cost or correlated native innovations
from disappearing through an unjustified factorization.

\section{Forecast-Visible Closed-Loop Approximation}
\label{sec:supp-closed-loop}

\subsection{Demand Is Predicted Once; Feedback Updates the Control State}

The Estimator maps pre-forecast history $H^{\rm pre}$ to a signed,
scale-resolved demand field
\begin{equation}
\widehat r=\bar R+B_rE_\psi(H^{\rm pre}).
\label{eq:supp-estimator-demand}
\end{equation}
The decoder and normalization are frozen before evaluation. The Estimator
predicts demand, not an action, and it runs once at initialization. At step
$t$, a deterministic mask removes leads that have already passed:
$\widehat r_{t:}=M_t\widehat r$. Live feedback does not rewrite
$\widehat r$; it enters the cost-to-go state separately.

The forecast-visible state is
\begin{equation}
s_t=(\widehat r_{t:},W_t,t,a_{t-1}^{\rm exec},m_{t-1},b_t^{\rm rem}),
\label{eq:supp-visible-state}
\end{equation}
where $W_t$ is the Pilot weather token, $m_{t-1}$ summarizes earlier Pilot
changes and executed actions, and $b_t^{\rm rem}$ is the exact remaining KL
ledger defined below. None of these quantities uses verifying weather.

\subsection{Action-Conditioned Future Cost}

The Q-Critic separates future multiscale score from low-rank moment
contributions. For each head $\ell\in\{\mathrm{skill},h\}$, zero advantage at
the native action is enforced by
\begin{equation}
Q_\phi^\ell(s,a)=V_\phi^\ell(s)+\widetilde A_\phi^\ell(s,a)
-\widetilde A_\phi^\ell(s,0).
\label{eq:supp-dueling-q}
\end{equation}
The pooled dual then produces the local cost-to-go
\begin{equation}
Q_{\phi,\xi}(s,a)=Q_\phi^{\rm skill}(s,a)
+(B_r^\top\xi)^\top Q_\phi^h(s,a).
\label{eq:supp-dual-q}
\end{equation}
The low-rank representation is a control approximation; the full reliability
functional remains defined in the original cell space. Action conditioning is
essential because the Controller needs
$\nabla_aQ_{\phi,\xi}(s_t,0)$, not only a state-value prediction.

\subsection{KL-Natural Proposal and Exact Feasibility}

For each fixed $(t,\mathcal G_t)$ represented by a sufficient state $s$, assume
the augmented fixed-support kernel family is differentiable in quadratic mean
at $a=0$, its KL is twice differentiable there, and the Q-Critic is
differentiable at the native action. Define exact normalized one-step cost by
\begin{equation}
d_t(a\mid\mathcal G_t)
=\nu_{\rm KL}^{-1}\KL\!\left(
K_{a,t}^{\rm aug}(\cdot\mid\mathcal G_t)
\Vert K_{0,t}^{\rm aug}(\cdot\mid\mathcal G_t)\right).
\label{eq:supp-exact-step-kl}
\end{equation}
Near $a=0$, this cost has the pointwise-in-$(s,t)$ expansion
\begin{equation}
\begin{aligned}
\bar D_{{\rm KL},t}(a\mid s)
&\equiv d_t(a\mid\mathcal G_t)\\
&=\tfrac12a^\top F_{\rm KL}(s)a+o(\lVert a\rVert^2),\\
F_{\rm KL}(s)
&=\left.\nabla_a^2\bar D_{{\rm KL},t}(a\mid s)\right|_{a=0}.
\end{aligned}
\label{eq:supp-local-kl}
\end{equation}
The Fisher matrix comes from the native kernel and Actuator parameterization,
not from the Q-Critic.

\begin{proposition}[Local KL-natural step]
Let $q_t=\nabla_aQ_{\phi,\xi}(s_t,0)$ and
$G_t=\lambda_{\rm KL}F_{\rm KL}(s_t)+\gamma I\succ0$. The unique minimizer of
the local unconstrained model
\begin{equation}
q_t^\top a+\tfrac12a^\top G_ta
\label{eq:supp-local-quadratic}
\end{equation}
is $a_t^{\rm nat}=-G_t^{-1}q_t$. If $\mathcal F_t^{\rm quad}$ is a nonempty
closed convex local feasible set, its constrained quadratic minimizer is
\begin{equation}
\begin{aligned}
\mathcal B_t^{\rm quad}
&=\left\{a:\tfrac12a^\top F_{\rm KL}(s_t)a
\leq b_t^{\rm step}\right\},\\
\mathcal F_t^{\rm quad}
&=\mathcal A_{\rm legal}^{\rm loc}\cap\mathcal B_t^{\rm quad},\\
a_t^{\rm quad}
&=\Pi_{\mathcal F_t^{\rm quad}}^{G_t}(a_t^{\rm nat}).
\end{aligned}
\label{eq:supp-local-projection}
\end{equation}
\label{prop:local-step}
\end{proposition}

\begin{proof}
The gradient of Equation~\eqref{eq:supp-local-quadratic} is $q_t+G_ta$;
positive definiteness gives the unique stationary minimizer
$-G_t^{-1}q_t$. Completing the square rewrites the objective as
\begin{equation}
\tfrac12\lVert a-a_t^{\rm nat}\rVert_{G_t}^2
-\tfrac12q_t^\top G_t^{-1}q_t.
\end{equation}
The second term is constant in $a$, so constrained minimization is exactly the
$G_t$-metric projection in Equation~\eqref{eq:supp-local-projection}. Closed
convexity and coercivity give existence and uniqueness. This proposition is a
conditional result for the declared local set; it does not assert convexity or
uniqueness after the nonlinear legal Actuator map.
\end{proof}

The Fisher quadratic is used only to propose a locally feasible action. After
the nonlinear legal Actuator map is applied, DIAL evaluates
Equation~\eqref{eq:supp-exact-step-kl} and backtracks toward $a=0$ until
\begin{equation}
\begin{aligned}
0\leq d_t(a_t^{\rm exec}\mid\mathcal G_t)
&\leq b_t^{\rm step}\leq b_t^{\rm rem},\\
b_{t+1}^{\rm rem}
&=b_t^{\rm rem}-d_t(a_t^{\rm exec}\mid\mathcal G_t).
\end{aligned}
\label{eq:supp-kl-ledger}
\end{equation}
The native action $a=0$ is always legal and has zero exact KL. If
$b_t^{\rm step}=0$ or a finite backtracking search does not find an acceptable
nonzero action, the Controller executes this native fallback.
With $b_0^{\rm rem}=\beta_{\rm path}$, induction gives
$\sum_t d_t(a_t^{\rm exec}\mid\mathcal G_t)\leq\beta_{\rm path}$ along every
executed path, and Proposition~\ref{prop:kl-chain-rule} gives the corresponding
expected normalized Path-KL bound. The reported realized Path-KL is the
evaluation average of that executed-kernel sum, not a pathwise probability
statement. All scores and responses are computed from the same executed
kernel, never from the pre-projection proposal.

For precision, the compact projected-action equation in the main paper is
schematic notation for this complete proposal-and-execution pipeline, not a
literal projection identity for the nonlinear executed action. Formally,
$a_t^{\rm quad}$ is the metric
projection proved in Proposition~\ref{prop:local-step}, whereas
$a_t^{\rm exec}$ is the result after the nonlinear legal map, exact-KL check,
and any backtracking. They coincide only when the exact legal set agrees with
the declared local set and no execution correction is required; no such
identity is used in the ledger proof.

\subsection{One Matched-Basis Actuator Library}

For stochastic source $k$, the general Actuator library may be written
abstractly as
\begin{equation}
Z'_{k,t+1}=\rho_{k,t}Z'_{k,t}
+\sigma_{k,t}B_k(\gamma_{k,t})\varepsilon_{k,t+1}.
\label{eq:supp-actuator-dynamics}
\end{equation}
Innovation scale, persistence, and the optional basis coordinate $\gamma$
describe one library of legal perturbations. They are not separate proposed
methods. The matched-basis principle assigns amplitude coordinates where the
missing degree of freedom is amplitude and configurational coordinates where
the missing degree of freedom is phase, balance, or regime structure. Every
extension is zero-initialized at the native kernel. The operating point and
all numerical audits in Sections~D--F use the main-paper evaluated instance:
$B_k(\gamma_{k,t})=B_k^0$ is fixed, and the available coordinates are the
source-matched innovation scales and persistence terms. Configurational bases
describe the larger method family but are not assigned the reported headline
values. In the source-tied restriction used in Table~\ref{tab:supp-actuator},
the intensity and persistence changes of each native source are coupled and
cannot be selected independently. The full evaluated Actuator retains both
coordinates independently for every source.

Hard feasibility precedes fidelity accounting. The legal projection preserves
the declared persistence range, covariance validity, support, normalization,
and variable balance. Path-KL then measures how far the resulting legal kernel
moves from the native kernel. A weak response from one frozen basis is evidence
about that basis, not proof that a trained matched-basis closure cannot work.

\subsection{When Live Feedback Can Improve on Stale Weather}

Let $\mathcal F_t^{\rm stale}$ be the information generated by pre-forecast
demand, the initial weather token, time, earlier executed actions, and the
budget ledger. Let
$\mathcal F_t^{\rm live}=\sigma(\mathcal F_t^{\rm stale},W_t)$ add the current
Pilot token.

\begin{proposition}[Value of live forecast information]
Let $\Pi_{\rm stale}$ and $\Pi_{\rm live}$ be the admissible policy classes
measurable with respect to $\mathcal F_t^{\rm stale}$ and
$\mathcal F_t^{\rm live}$, with the same actions and constraints. Compare both
classes under a common predecision probability law $\mathbb P_t$, and suppose
$L(\pi)\in L^1(\mathbb P_t)$ for every policy being compared. Then
\begin{equation}
\inf_{\pi\in\Pi_{\rm live}}\E[L(\pi)]
\leq
\inf_{\pi\in\Pi_{\rm stale}}\E[L(\pi)].
\label{eq:supp-feedback-value}
\end{equation}
For strictness, additionally suppose that a feasible selector
$a^*\in\Pi_{\rm live}$ attains the unique conditional minimum of remaining
expected cost almost surely over the same admissible continuation class, a
stale optimum $\pi_{\rm stale}^*\in\Pi_{\rm stale}$ is attained, and
$\mathbb P_t\{\pi_{\rm stale}^*\ne a^*\}>0$. Then the inequality is strict.
\label{prop:feedback-value}
\end{proposition}

\begin{proof}
Every stale-information policy is also live-information measurable because
$\mathcal F_t^{\rm stale}\subseteq\mathcal F_t^{\rm live}$. Hence
$\Pi_{\rm stale}\subseteq\Pi_{\rm live}$, which proves the weak inequality.
Under the additional hypotheses, uniqueness makes the conditional regret of
$\pi_{\rm stale}^*$ nonnegative everywhere and strictly positive on the
positive-probability set where it differs from the feasible live selector
$a^*$. The common predecision law makes this a comparison on one probability
space. Integrability makes the regret expectation finite, and a nonnegative
random variable that is positive with positive probability has strictly
positive expectation. The attained stale value is therefore strictly larger
than the live value.
\end{proof}

The proposition is conditional, not an unconditional superiority theorem for
trained networks. It motivates two empirical checks: the preferred action must
change with live forecast state, and the Pilot token must predict that change
beyond stale information.

\section{Experimental Protocol and Comparability Rules}
\label{sec:supp-protocol}

\subsection{Forecast Scope and Frozen Evaluation Grid}

All comparisons use the same frozen FCN3 forecast task and ERA5 verification
operator as the main paper. The reported fields are geopotential, temperature,
zonal wind, and meridional wind at 500 and 850 hPa: Z500, Z850, T500, T850,
U500, U850, V500, and V850. Each field is evaluated in eight training-frozen
spectral bands and at 28 six-hour leads through 168 h. The three lead windows
are 6--48 h, 54--96 h, and 102--168 h. Two validation-frozen strata of a
forecast-visible large-scale flow index support the state-conditional audit.
They are low/high index strata, not a discovered synoptic taxonomy.

All learned components, normalizers, spectral projections, metric weights,
action constraints, and comparison rules are frozen before the reported
evaluation. Scored members have equal weight. The Pilot uses an independent
random stream and is excluded from every ensemble score.

\subsection{Method Boundaries}

The comparison roster separates methods by where they intervene.

\paragraph{Static output reference.}
Static-SP applies one training-frozen, mean-preserving spread action for each
variable--band pair. The action is shared across all leads and weather states.

\paragraph{Adaptive output-only methods.}
Cond-SP performs weather-conditioned scale-resolved output calibration. PoET
processes the completed member set, while generative multivariate
postprocessing samples a completed output law
\cite{benbouallegue2024poet,chen2024generativepostprocessing}. These methods
can modify final members but do not execute native-noise actions during the
forecast.

\paragraph{Generation-time controllers.}
DRaFT-K-style, KL-PG, and DIAL act in the same native-noise coordinates and use
the same legal action boundary and ex-ante Path-KL budget. DRaFT-K-style uses
truncated reward backpropagation, KL-PG uses a KL-regularized policy-gradient
controller, and DIAL uses the action-conditioned critic and projected
KL-natural update \cite{clark2024draft,black2024ddpo,fan2023dpok}.

Path-KL is defined only for generation-time controllers. Output distortion is
defined only for output interventions. Aggregate signed action alignment is
reported only for controllers because an output remapping does not expose a
like-for-like sequence of native-kernel actions. When an output method appears
in a signed-response table, its separately named spread-response direction
fraction diagnoses the final-member change and is not treated as controller
action alignment.

\subsection{Metrics and Aggregation Orders}

The primary metric is Native-normalized MS-afCRPS:
\begin{equation}
\mathrm{nMS\text{-}afCRPS}(m)
=\frac{\mathcal S_{\rm MS}(m)}{\mathcal S_{\rm MS}(\mathrm{Native})}.
\label{eq:supp-nms}
\end{equation}
Full-field fCRPSS is computed separately from full-field fair CRPS. Ensemble
mean RMSE skill, RMS($C$), $\jsecond/\epsilon_{\rm rel}$, a fixed-pair
variogram skill, and realized Path-KL are complementary audits. Lower values
are better for nMS-afCRPS, RMS($C$), and $\jsecond/\epsilon_{\rm rel}$; higher
values are better for the three skill scores.

Two noncommuting aggregation orders are used deliberately. A field--lead curve
first averages normalized scores over bands and state strata, then converts
the mean to Native-relative skill. A scorecard cell first computes the
cellwise DIAL-to-KL-PG ratio, then averages relative reduction inside one
field--band--lead-window block. The two views therefore answer different
questions and need not be numerically identical after marginalization.
The mechanism tables retain a third registered distinction. Their
primary-weighted long-window skill uses the full-field/band weights of
Equation~\eqref{eq:supp-nms}, whereas the Actuator scale and lead diagnostics
average normalized projected-cell scores with equal cell weight. The latter
isolates where a coordinate restriction acts; it is not a second estimate of
the primary long-window endpoint.

For signed response, a cell is materially underdispersed when
$C_{\rm native}\geq0.15$ and materially overdispersed when
$C_{\rm native}\leq-0.15$. For a controller, action alignment is the fraction
for which $C_{\rm native}\,\delta_s>0$. For an output method, the same sign
calculation is called the spread-response direction fraction. State and lead
sign-change subsets are defined before method comparison. The document reports
point comparisons only; an inferential claim would require paired
forecast-date resampling and a pre-specified multiplicity procedure.

\subsection{Joint Consistency Rules}

Every cell-derived aggregate, lead curve, signed subset, and equal-cell
Actuator diagnostic is computed from one field--band--lead--state tensor.
Registered main-paper and ablation anchors retain their frozen primary
aggregation rather than being relabeled as equal-cell diagnostics. Within each
method and aggregation view, the following identities are enforced jointly:
\begin{equation}
\begin{aligned}
\afcrps_{0.95}&=0.95\fcrps+0.05\crps,\\
C_{\rm after}&=C_{\rm native}+2\delta_e-2\delta_s.
\end{aligned}
\label{eq:supp-joint-identities}
\end{equation}
RMSE is recomputed from the same error amplitudes, and
$\jsecond/\epsilon_{\rm rel}$ is recomputed from the same residual field and
its registered auxiliary moment contribution. Existing main-paper result and
ablation values are inherited without alteration. No selected local case is
used to estimate an aggregate response frequency.

\raggedbottom
\section{Extended Results Across Variables, Scales, and Leads}
\label{sec:supp-extended-results}

\subsection{Aggregate Point Comparison with Static-SP}

Table~\ref{tab:supp-aggregate} separates reliability, mean skill, and
native-law fidelity. DIAL has the lowest primary score, RMS($C$), and pooled
reliability ratio while using less realized Path-KL than the two
generation-time comparators at the registered budget. Generative
postprocessing retains the largest ensemble-mean RMSE gain. It also retains
the strongest fixed-pair variogram skill (2.681\%, compared with 0.728\% for
DIAL), showing that the primary-score and fixed-pair fidelity rankings capture
complementary behavior. Static-SP improves the Native score but remains above
the reliability threshold, consistent with the value of lead- and
state-dependent actions.

\begin{table}[H]
\centering
\footnotesize
\setlength{\tabcolsep}{1.8pt}
\renewcommand{\arraystretch}{1.05}
\textit{(a) Primary score and forecast skill}\\[2pt]
\begin{tabular}{@{}lrrr@{}}
\toprule
Method & nMS-afCRPS $\downarrow$ & fCRPSS (\%) $\uparrow$
& RMSE gain (\%) $\uparrow$ \\
\midrule
Native FCN3 & 1.0000 & 0.000 & 0.000 \\
Static-SP & 0.9827 & 1.431 & 0.000 \\
Cond-SP & 0.9681 & 2.737 & 0.018 \\
PoET & 0.9453 & 4.824 & 1.274 \\
Gen.\ postproc. & 0.9370 & 5.913 & 2.413 \\
DRaFT-K-style & 0.9345 & 5.663 & 1.537 \\
KL-PG & 0.9320 & 6.050 & 1.780 \\
DIAL & 0.9173 & 7.543 & 2.117 \\
\bottomrule
\end{tabular}
\par\vspace{5pt}
\textit{(b) Reliability and native-law fidelity}\\[2pt]
\begin{tabular}{@{}lrrr@{}}
\toprule
Method & RMS($C$) $\downarrow$ & $\jsecond/\epsilon_{\rm rel}$ $\downarrow$
& Path-KL \\
\midrule
Native FCN3 & 0.4353 & 1.480 & 0.0000 \\
Static-SP & 0.3751 & 1.236 & -- \\
Cond-SP & 0.3207 & 1.104 & -- \\
PoET & 0.2491 & 0.928 & -- \\
Gen.\ postproc. & 0.2278 & 0.817 & -- \\
DRaFT-K-style & 0.2436 & 0.892 & 0.0266 \\
KL-PG & 0.2310 & 0.808 & 0.0273 \\
DIAL & 0.1794 & 0.591 & 0.0259 \\
\bottomrule
\end{tabular}
\caption{Extended aggregate comparison. DIAL has the lowest
nMS-afCRPS, RMS($C$), and $\jsecond/\epsilon_{\rm rel}$, and uses less Path-KL
than both generation-time comparators; generative postprocessing retains the
largest RMSE gain. Static-SP is a mean-preserving training-frozen correction
reported as a supplement-only control. For the seven methods shared with the
main paper, nMS-afCRPS, fCRPSS, RMSE gain, and RMS($C$) reproduce the same
validation-frozen operating points; panel (b) adds the corresponding pooled-
reliability and Path-KL audits. Path-KL is reported for Native and
generation-time controllers; em dashes identify output interventions.}
\label{tab:supp-aggregate}
\end{table}

\subsection{Complete Eight-Field Lead Curves}

Figure~\ref{fig:supp-full-lead-curves} extends the representative main-paper
curves to all eight pressure-level fields. DIAL has the largest average skill
in every field, but its margin varies with lead. Z500 dips at short range before
widening at medium and long leads; U850 and V850 show additional lead structure,
whereas Z850, T500, U500, and V500 vary more smoothly. These field-level averages
are complemented by signed audits that retain local DIAL losses.

\begin{figure*}[!t]
\centering
\includegraphics[width=0.90\textwidth]{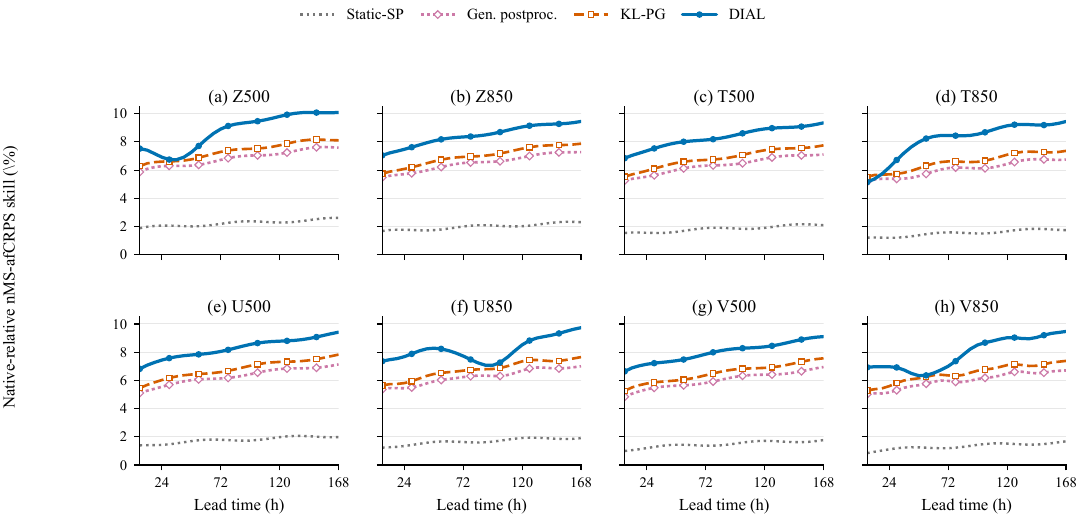}
\caption{Native-relative nMS-afCRPS skill across all eight fields and 28
six-hour leads. DIAL has the highest lead-averaged skill in every field, while
its margin varies with lead. Static-SP is a training-frozen output correction;
generative postprocessing has the lowest aggregate nMS-afCRPS among output-side
comparators, and KL-PG has the lowest aggregate nMS-afCRPS among generation-time
baselines.
Positive values favor the displayed method over Native FCN3.}
\label{fig:supp-full-lead-curves}
\end{figure*}

\section{Signed Reliability Response Across State and Lead}
\label{sec:supp-signed-response}

\subsection{Bidirectional Response on Material Gaps}

\begin{table}[H]
\centering
\footnotesize
\setlength{\tabcolsep}{5.0pt}
\renewcommand{\arraystretch}{1.05}
\textit{(a) Underdispersed: $C_{\rm native}\geq0.15$}\\[2pt]
\begin{tabular}{@{}lrrr@{}}
\toprule
Method & nMS-afCRPS & $\Delta\log$ spread & RMS($C$) \\
\midrule
Static-SP & 0.9692 & +0.0315 & 0.4150 \\
Gen.\ postproc. & 0.9256 & +0.0858 & 0.2437 \\
KL-PG & 0.9329 & +0.0769 & 0.2794 \\
DIAL & 0.9199 & +0.1059 & 0.2268 \\
\bottomrule
\end{tabular}
\par\vspace{5pt}
\textit{(b) Overdispersed: $C_{\rm native}\leq-0.15$}\\[2pt]
\begin{tabular}{@{}lrrr@{}}
\toprule
Method & nMS-afCRPS & $\Delta\log$ spread & RMS($C$) \\
\midrule
Static-SP & 0.9957 & -0.0354 & 0.4413 \\
Gen.\ postproc. & 0.9492 & -0.1331 & 0.2782 \\
KL-PG & 0.9316 & -0.1444 & 0.2449 \\
DIAL & 0.9157 & -0.1809 & 0.1767 \\
\bottomrule
\end{tabular}
\caption{Bidirectional response on material Native reliability gaps.
DIAL adds effective spread on underdispersed cells ($+0.1059$) and suppresses
it on overdispersed cells ($-0.1809$), attaining the lowest nMS-afCRPS and
RMS($C$) in both groups. The groups contain 1464 and 1267
field--band--lead--state cells, respectively.}
\label{tab:supp-bidirectional}
\end{table}

The direction of DIAL's mean response reverses with the Native gap:
$+0.1059$ on underdispersed cells and $-0.1809$ on overdispersed cells.
Equation~\eqref{eq:supp-c-response-identity} then gives lower residual
RMS($C$) than KL-PG and the output-side comparators in both material groups.
Each group is scored on its own. Improvement therefore cannot arise by
cancelling opposing signed residuals.

\subsection{Weather-State and Lead Sign Changes}

\begin{table}[H]
\centering
\scriptsize
\setlength{\tabcolsep}{2.6pt}
\renewcommand{\arraystretch}{1.05}
\textit{(a) Primary score by sign-change subset}\\[2pt]
\begin{tabular}{@{}lrrrrr@{}}
\toprule
Subset & Cells & Static-SP & Gen.\ postproc. & KL-PG & DIAL \\
\midrule
Stable & 2218 & 0.9732 & 0.9304 & 0.9320 & 0.9176 \\
State flip & 470 & 0.9976 & 0.9471 & 0.9323 & 0.9167 \\
Lead flip & 488 & 0.9978 & 0.9464 & 0.9289 & 0.9128 \\
Both flips & 408 & 0.9989 & 0.9498 & 0.9353 & 0.9218 \\
\bottomrule
\end{tabular}
\par\vspace{5pt}
\textit{(b) Signed-response diagnostics}\\[2pt]
\begin{tabular}{@{}lrrr@{}}
\toprule
Subset & \shortstack{DIAL margin\\vs. best static/output (\%)}
& \shortstack{Gen. spread\\response fraction}
& \shortstack{DIAL action\\alignment} \\
\midrule
Stable & 1.377 & 0.9558 & 0.8454 \\
State flip & 3.214 & 0.4936 & 0.9340 \\
Lead flip & 3.555 & 0.6475 & 0.9344 \\
Both flips & 2.945 & 0.5466 & 0.9412 \\
\bottomrule
\end{tabular}
\caption{Primary-score and signed-response behavior stratified by
whether the Native reliability-gap sign changes with weather state, lead, or
both. DIAL's margin over the best static or output-only comparator rises from
1.377\% on stable-sign cells to 2.945--3.555\% on sign-changing subsets, where
its action alignment remains above 0.93. Method columns report nMS-afCRPS; the
generative response fraction records whether the final-member spread change
has the required sign, and DIAL alignment records whether the Controller action
has that sign. nMS-afCRPS supplies the proper-score comparison, while the last
two columns diagnose response direction.}
\label{tab:supp-sign-change}
\end{table}

DIAL's primary-score margin over the best static or output-only method is
1.377\% on stable-sign cells; count-weighted pooling gives 3.256\% across the
three sign-changing groups, 1.879 percentage points larger. Generative
postprocessing's spread-response fraction exceeds DIAL's action alignment on
stable cells but falls under sign changes, whereas DIAL's alignment remains
above 0.93 in all three groups. This isolates live response to a changing
correction direction, while the stable-cell comparison remains method- and
diagnostic-specific.

\subsection{Selected Signed Challenge Cells}

The U500 pair makes the lead-dependent reversal explicit: the same variable
and band require increased spread at 48 h and suppressed spread at 132 h.
Table~\ref{tab:supp-signed-cases} also keeps two counterexamples where DIAL
moves in the wrong direction. Together with the group averages, they document
both typical correction and retained local failures.

\begin{table}[H]
\centering
\scriptsize
\setlength{\tabcolsep}{2.2pt}
\renewcommand{\arraystretch}{1.02}
\textit{(a) Intended-direction responses}\\[2pt]
\begin{minipage}[t]{\columnwidth}
\centering
\begin{tabular}{@{}llrrrrr@{}}
\toprule
Field & Band, lead & $C_{\rm N}$ & $\Delta s_{\rm D}$
& $\Delta e_{\rm D}$ & $C_{\rm D}$ & $C_{\rm G}$ \\
\midrule
Z500 & B2, 72 h & +0.347 & +0.126 & -0.006 & +0.083 & +0.118 \\
T850 & B5, 120 h & +0.281 & +0.091 & -0.004 & +0.091 & +0.076 \\
U500 & B3, 48 h & +0.224 & +0.071 & +0.003 & +0.088 & +0.106 \\
U500 & B3, 132 h & -0.196 & -0.060 & -0.002 & -0.080 & -0.112 \\
V850 & B7, 96 h & -0.318 & -0.101 & +0.004 & -0.108 & -0.132 \\
T500 & B8, 168 h & -0.247 & -0.069 & +0.005 & -0.099 & -0.081 \\
Z850 & B1, 24 h & +0.183 & +0.046 & +0.001 & +0.093 & +0.069 \\
V500 & B4, 144 h & +0.296 & +0.086 & -0.005 & +0.114 & +0.146 \\
\bottomrule
\end{tabular}
\end{minipage}
\par\vspace{5pt}
\textit{(b) Retained wrong-direction responses}\\[2pt]
\begin{minipage}[t]{\columnwidth}
\centering
\begin{tabular}{@{}llrrrrr@{}}
\toprule
Field & Band, lead & $C_{\rm N}$ & $\Delta s_{\rm D}$
& $\Delta e_{\rm D}$ & $C_{\rm D}$ & $C_{\rm G}$ \\
\midrule
T850 & B2, 54 h & -0.172 & +0.017 & -0.003 & -0.212 & -0.074 \\
Z500 & B8, 30 h & +0.164 & -0.011 & +0.002 & +0.190 & +0.092 \\
\bottomrule
\end{tabular}
\end{minipage}
\captionsetup{font=footnotesize,skip=1pt}
\caption{Selected signed challenge cells satisfying
Equation~\eqref{eq:supp-c-response-identity}. DIAL moves in the intended
direction in the eight panel-(a) examples; panel (b) retains two local
wrong-direction responses. Generative postprocessing has lower $|C|$ in five
of ten rows. $C_{\rm N}$, $C_{\rm D}$, and $C_{\rm G}$ are Native, DIAL, and
generative-postprocessing reliability gaps; $\Delta s_{\rm D}$ and
$\Delta e_{\rm D}$ are DIAL changes in log spread and log error. These examples
illustrate local response mechanisms and retained failures; the exhaustive
partition in Table~\ref{tab:supp-sign-change} supplies aggregate alignment.}
\label{tab:supp-signed-cases}
\end{table}

\section{Budget Robustness and Mechanism Audits}
\label{sec:supp-mechanism}

\subsection{Registered Budget and Sensitivity Grid}

Table~\ref{tab:supp-budget} shows a constrained optimum at the registered
operating point. Increasing the budget to the pre-specified value
$\beta_{\rm path}=0.030$ improves the primary score, pooled reliability, and
variogram audit together. Raising it further to $0.045$ only reduces
$\jsecond/\epsilon_{\rm rel}$ while worsening both nMS-afCRPS and variogram
skill. We therefore retain $0.030$ as the validation-frozen operating point.

\begin{table}[H]
\centering
\footnotesize
\setlength{\tabcolsep}{2.8pt}
\renewcommand{\arraystretch}{1.03}
\begin{tabular}{@{}rrrrr@{}}
\toprule
$\beta_{\rm path}$ & \shortstack{Realized\\KL} & nMS-afCRPS
& $\jsecond/\epsilon_{\rm rel}$ & \shortstack{Variogram\\skill (\%)} \\
\midrule
0.000 & 0.0000 & 1.0000 & 1.480 & 0.000 \\
0.006 & 0.0048 & 0.9603 & 1.193 & 0.158 \\
0.012 & 0.0098 & 0.9415 & 0.951 & 0.361 \\
0.020 & 0.0165 & 0.9267 & 0.728 & 0.579 \\
0.030 & 0.0259 & 0.9173 & 0.591 & 0.728 \\
0.045 & 0.0387 & 0.9201 & 0.566 & 0.214 \\
\bottomrule
\end{tabular}
\caption{DIAL sensitivity across the validation-frozen ex-ante Path-KL grid.
The pre-specified $\beta_{\rm path}=0.030$ gives the lowest nMS-afCRPS and the
highest variogram skill; raising the budget to $0.045$ only lowers
$\jsecond/\epsilon_{\rm rel}$ while degrading both metrics. Other rows are
sensitivity points around this registered operating point.}
\label{tab:supp-budget}
\end{table}

\subsection{Restrictions of the Matched-Basis Actuator}

The full evaluated coordinate set gives the best primary score, RMSE gain, and
pooled reliability ratio at comparable realized KL. Amplitude-only control
retains 97.9\% of full DIAL's small-scale skill but 77.5\% of its large-scale
skill, consistent with amplitude repair being more complete at smaller scales.
Persistence-only control retains substantial long-lead skill while giving
lower small-scale and aggregate skill. Coupling intensity and persistence
within each source recovers complementary scale and lead benefits: the
source-tied restriction improves nMS-afCRPS, large-scale skill, and long-lead
skill over either single-coordinate restriction, with the full evaluated
source-matched coordinate set retaining the strongest overall performance.

\subsection{Forecast-Visible Information and Memory}

Each information restriction degrades the primary score and long-lead skill.
Removing recurrent memory produces the smallest loss, while removing the Pilot
or both information sources produces substantially larger losses. The
main-paper matched stale-weather ablation retains case-specific demand and
repeats the initial weather token. Here, No Pilot retains that demand but
replaces every weather token with its training mean, while the first row also
replaces demand with its training climatology. The four rows establish
directional component audits; estimating an Estimator-by-feedback interaction
requires paired uncertainty for the complete factorial contrast.

\begin{table*}[!t]
\begin{minipage}[t]{0.485\textwidth}
\centering
\scriptsize
\setlength{\tabcolsep}{3.0pt}
\renewcommand{\arraystretch}{1.04}
\textit{(a) Aggregate score, mean skill, reliability, and fidelity}\\[2pt]
\begin{tabular}{@{}lrrrr@{}}
\toprule
Actuator restriction & nMS-afCRPS & RMSE gain (\%)
& $\jsecond/\epsilon_{\rm rel}$ & Path-KL \\
\midrule
Amplitude only & 0.9258 & 1.905 & 0.704 & 0.0243 \\
Persistence only & 0.9321 & 1.442 & 0.803 & 0.0228 \\
Source-tied & 0.9234 & 2.004 & 0.654 & 0.0261 \\
Full evaluated & 0.9173 & 2.117 & 0.591 & 0.0259 \\
\bottomrule
\end{tabular}
\par\vspace{5pt}
\textit{(b) Native-relative equal-cell scale and lead skill}\\[2pt]
\begin{tabular}{@{}lrrrr@{}}
\toprule
Restriction & \shortstack{Large-scale\\skill (\%)}
& \shortstack{Mid-scale\\skill (\%)}
& \shortstack{Small-scale\\skill (\%)}
& \shortstack{Long-lead\\skill (\%)} \\
\midrule
Amplitude only & 6.632 & 7.601 & 7.764 & 8.236 \\
Persistence only & 7.509 & 7.061 & 6.039 & 7.873 \\
Source-tied & 7.949 & 7.790 & 7.338 & 8.476 \\
Full evaluated & 8.556 & 8.413 & 7.929 & 9.084 \\
\bottomrule
\end{tabular}
\captionof{table}{Restrictions of the evaluated source-matched Actuator. At
comparable Path-KL, the full evaluated Actuator has the lowest nMS-afCRPS and
$\jsecond/\epsilon_{\rm rel}$ and the highest RMSE gain and every displayed
scale and lead skill. Scale columns pool B1--B2, B3--B5, and B6--B8; long lead
covers 102--168 h. Each row evaluates the DIAL Estimator--Controller framework
under a different admissible subset of Actuator coordinates.}
\label{tab:supp-actuator}
\end{minipage}\hfill
\begin{minipage}[t]{0.485\textwidth}
\centering
\scriptsize
\setlength{\tabcolsep}{2.0pt}
\renewcommand{\arraystretch}{1.04}
\textit{(a) Aggregate performance}\\[2pt]
\begin{tabular}{@{}lrrr@{}}
\toprule
Information variant & nMS-afCRPS & RMSE gain (\%)
& $\jsecond/\epsilon_{\rm rel}$ \\
\midrule
Climatology + stale wx. & 0.9595 & 0.628 & 0.946 \\
No Pilot & 0.9509 & 0.923 & 0.852 \\
No recurrent memory & 0.9224 & 1.904 & 0.628 \\
Full DIAL & 0.9173 & 2.117 & 0.591 \\
\bottomrule
\end{tabular}
\par\vspace{5pt}
\textit{(b) Native-law fidelity and long-lead skill}\\[2pt]
\begin{tabular}{@{}lrr@{}}
\toprule
Information variant & Realized Path-KL
& \shortstack{Primary-weighted\\102--168 h skill (\%)} \\
\midrule
Climatology + stale wx. & 0.0229 & 3.814 \\
No Pilot & 0.0246 & 4.672 \\
No recurrent memory & 0.0262 & 7.964 \\
Full DIAL & 0.0259 & 8.617 \\
\bottomrule
\end{tabular}
\captionof{table}{Supplementary information-source audit. Full DIAL has the
lowest nMS-afCRPS and $\jsecond/\epsilon_{\rm rel}$ and the highest RMSE gain
and primary-weighted long-lead skill. Removing recurrent memory causes the
smallest degradation; removing live forecast information causes larger
losses. The first row combines training-climatological demand with the initial
weather token; No Pilot uses the training-mean weather token; and No recurrent
memory resets Q-Critic memory at each step.}
\label{tab:supp-information}
\end{minipage}
\end{table*}

\FloatBarrier

\bibliography{references}

\end{document}